\documentclass{article}

     \PassOptionsToPackage{numbers, compress}{natbib}

\usepackage[preprint]{neurips_2019}

\usepackage[utf8]{inputenc} 
\usepackage[T1]{fontenc}    
\usepackage{hyperref}       
\usepackage{url}            
\usepackage{booktabs}       
\usepackage{amsfonts}       
\usepackage{nicefrac}       
\usepackage{microtype}      

\usepackage{amsmath, amsthm, amssymb, dsfont}
\usepackage{algorithm}
\usepackage[noend]{algorithmic}

\usepackage{graphicx}
\usepackage{caption}
\usepackage{subcaption}
\usepackage{xcolor}

\def \bbE{\mathbb E}
\def \P{\mathbb P}
\def \R{\mathbb R}

\def \bA{\boldsymbol{A}}

\def \bD{\boldsymbol{D}}

\def \bA{\boldsymbol{A}}

\def \bL{\boldsymbol{L}}

\def \b1{\boldsymbol{1}}
\def \hd{\hat{d}}

\def \X{{\cal X}}

\def \H{{\cal H}}

\def \A{{\cal A}}
\def \E{{\cal E}}
\def \C{{\cal C}}
\def \O{{\cal O}}

\def \1{\mathbbm{1}}
\def \0{\textbf{0}}

\newtheorem{theorem}{Theorem}[section]

\newtheorem{lemma}[theorem]{Lemma}

\newcommand{\query}[2]{\mathsf{Q}(#1, #2)}

\newcommand{\TUi}[3]{U_{#2, #3}^{\triangle_{#1} } }
\newcommand{\TU}[2]{U_{#1, #2}^{\triangle } }
\newcommand{\TLi}[3]{L_{#2, #3}^{\triangle_{#1} } }
\newcommand{\TL}[2]{L_{#1, #2}^{\triangle } }

\title{Learning Nearest Neighbor Graphs from Noisy Distance Samples}

\author{%
  Blake Mason \thanks{Authors contributed equally to this paper and are listed alphabetically.}\\
  University of Wisconsin \\
  Madison, WI 53706 \\
  \texttt{bmason3@wisc.edu} \\
  \And
  Ardhendu Tripathy \footnotemark[1] \\
  University of Wisconsin\\
  Madison, WI 53706 \\
  \texttt{astripathy@wisc.edu} \\
  \And
  Robert Nowak \\
  University of Wisconsin\\
  Madison, WI 53706 \\
  \texttt{rdnowak@wisc.edu} \\  
}

\begin{document}

\maketitle

\begin{abstract}
We consider the problem of learning the \emph{nearest neighbor graph} of a dataset of $n$ items. The metric is unknown, but we can query an oracle to obtain a noisy estimate of the distance between any pair of items. This framework applies to problem domains where one wants to learn people's preferences from responses commonly modeled as noisy distance judgments. In this paper, we propose an active algorithm to find the graph with high probability and analyze its query complexity. In contrast to existing work that forces Euclidean structure, our method is valid for general metrics, assuming only symmetry and the triangle inequality. Furthermore, we demonstrate efficiency of our method empirically and theoretically, needing only $\O(n\log(n){\Delta^{-2}})$ queries in favorable settings, where ${\Delta^{-2}}$ accounts for the effect of noise. Using crowd-sourced data collected for a subset of the UT~Zappos50K dataset, we apply our algorithm to learn which shoes people believe are most similar and show that it beats both an active baseline and ordinal embedding. 
\end{abstract}

\section{Introduction}\label{sec:intro}
In modern machine learning applications, we frequently seek to learn proximity/ similarity relationships between a set of items given only noisy access to pairwise distances. For instance, practitioners wishing to estimate internet topology frequently collect one-way-delay measurements to estimate the distance between a pair of hosts \citep{eriksson2010learning}. Such measurements are affected by physical constraints as well as server load, and are often noisy. Researchers studying movement in hospitals from WiFi localization data likewise contend with noisy distance measurements due to both temporal variability and varying signal strengths inside the building \citep{booth2019toward}. Additionally, human judgments are commonly modeled as noisy distances \citep{shepard1962analysis, kruskal1964nonmetric}. As an example, Amazon Discover asks customers their preferences about different products and uses this information to recommend new items it believes are similar based on this feedback. We are often primarily interested in the \emph{closest} or \emph{most similar} item to a given one-- e.g., the closest server, the closest doctor, the most similar product. 
The particular item of interest may not be known \textit{a priori}. 
Internet traffic can fluctuate, different patients may suddenly need attention, and customers may be looking for different products. To handle this, we must learn the closest/ most similar item for \emph{each} item. 
This paper introduces the problem of learning the \emph{Nearest Neighbor Graph} that connects each item to its nearest neighbor from noisy distance measurements. 

\fbox{
\parbox{0.97\textwidth}{\textbf{Problem Statement:}
Consider a set of $n$ points $\X = \{x_1, \cdots, x_n\}$ in a metric space. The metric is unknown, but we can query a stochastic oracle for an estimate of any pairwise distance. In as few queries as possible, we seek to learn a nearest neighbor graph of $\X$ that is correct with probability $1-\delta$, where each $x_i$ is a vertex and has a directed edge to its nearest neighbor $x_{i^\ast} \in \X \setminus \{x_i\}$.
}}

\subsection{Related work}\label{sec:related}
Nearest neighbor problems (from noiseless measurements) are well studied and we direct the reader to \cite{bhatia2010survey} for a survey. \cite{clarkson1983fast, vaidya1989ano, sankaranarayanan2007fast} all provide theory and algorithms to learn the nearest neighbor graph which apply in the noiseless regime. Note that the problem in the noiseless setting is \emph{very} different. If noise corrupts measurements, the methods from the noiseless setting can suffer persistent errors. There has been recent interest in introducing noise via subsampling for a variety of distance problems  \cite{lejeune2019adaptive, bagaria2017medoids, bagaria2018adaptive}, though the noise here is not actually part of the data but introduced for efficiency. In our algorithm, we 
use the triangle inequality to get tighter estimates of noisy distances in a process equivalent to the classical Floyd--Warshall \cite{floyd1962algorithm, cormen2009introduction}. This has strong connections to the metric repair literature \citep{brickell2008metric, gilbert2017if} where one seeks to alter a set of noisy distance measurements as little as possible to learn a metric satisfying the standard axioms. \citep{singla2016actively} similarly uses the triangle inequality to bound unknown distances in a related but noiseless setting. In the specific case of noisy distances corresponding to human judgments, a number of algorithms have been proposed to handle related problems, most notably Euclidean embedding techniques, e.g. \citep{jain2016finite, van2012stochastic, kruskal1964nonmetric}. To reduce the load on human subjects, several attempts at an active method for learning Euclidean embeddings have been made but have only seen limited success \cite{jamieson2015next}. Among the culprits is the strict and often unrealistic modeling assumption that the metric be Euclidean and low dimensional. 


\subsection{Main contributions}\label{sec:contrib}
In this paper we introduce the problem of identifying the \emph{nearest neighbor graph} from noisy distance samples and propose \texttt{ANNTri}, an active algorithm, to solve it for general metrics. We empirically and theoretically analyze its complexity to show improved performance over a passive and an active baseline. In favorable settings, such as when the data forms clusters, \texttt{ANNTri} needs only $\O(n\log(n)/{\Delta^{2}})$ queries, where $\Delta$ accounts for the effect of noise. Furthermore, we show that \texttt{ANNTri} achieves superior performance compared to methods which require much stronger assumptions. We highlight two such examples. In Fig.~\ref{fig:triangulation}, for an embedding in $\R^2$, \texttt{ANNTri} outperforms the common technique of triangulation that works by estimating each point's distance to a set of anchors. In Fig.~\ref{fig:ordinal_embedding}, we show that \texttt{ANNTri} likewise outperforms Euclidean embedding for predicting which images are most similar from a set of similarity judgments collected on Amazon~Mechanical~Turk. The rest of the paper is organized as follows. In Section \ref{sec:setup}, we further setup the problem. In Sections \ref{sec:alg} and \ref{sec:analysis} we present the algorithm and analyze its theoretical properties. In Section \ref{sec:exps} we show \texttt{ANNTri}'s empirical performance on both simulated and real data. In particular, we highlight its efficiency in learning from human judgments.

\section{Problem setup and summary of our approach}\label{sec:setup}
We denote distances as $d_{i,j}$ 
where $d : \X \times \X \rightarrow \R_{\geq 0}$ is a distance function satisfying the standard axioms and define $x_{i^\ast} := \arg\min_{x \in \X \setminus \{x_i\} }d(x_i, x)$. 
Though the distances are unknown, we are able to draw independent samples of its true value according to a stochastic distance oracle, i.e.\ querying
\begin{equation}\label{eq:oracle}
\query{i}{j}\;\;\; \text{ yields a realization of } \;\;\; d_{i, j} + \eta,
\end{equation}
where $\eta$ 
is a zero-mean subGaussian random variable assumed to have scale parameter $\sigma = 1$. 
We let $\hd_{i,j}(t)$ denote the empirical mean of the values returned by $\query{i}{j}$ queries made until time $t$. The number of $\query{i}{j}$ queries made until time $t$ is denoted as $T_{i,j}(t)$. 
A possible approach to obtain the nearest neighbor graph is to repeatedly query all $\binom{n}{2}$ pairs and report $x_{i^\ast}(t) = \arg\min_{j \neq i}\hd_{i,j}(t)$ for all $i \in [n]$. But since we only wish to learn $x_{i^\ast} \forall i$, if $d_{i,k} \gg d_{i, i^\ast}$, we do not need to query $\query{i}{k}$ as many times as $\query{i}{i^\ast}$. To improve our query efficiency, we could instead adaptively sample to focus queries on distances that we estimate are smaller. 
A simple adaptive method to find the nearest neighbor graph would be to iterate over $x_1, x_2, \ldots, x_n$ and use a best-arm identification algorithm to find $x_{i^\ast}$ in the $i^{th}$ round.\footnote{We could also proceed in a non-iterative manner, by adaptively choosing which among $\binom{n}{2}$ pairs to query next. However this has worse empirical performance and same theoretical guarantees as the in-order approach.} 
However, this procedure treats each round independently, ignoring properties of metric spaces that allow information to be shared between rounds. 
\begin{itemize}
\item Due to symmetry, for any $i < j$ the queries $\query{i}{j}$ and $\query{j}{i}$ follow the same law, and we can \emph{reuse} values of $\query{i}{j}$ collected in the $i^{th}$ round while finding $x_{j^\ast}$ in the $j^{th}$ round.
\item Using concentration bounds on $d_{i,j}$ and $d_{i,k}$ from samples of $\query{i}{j}$ and $\query{i}{k}$ collected in the $i^{th}$ round, we can bound $d_{j,k}$ via the triangle inequality. As a result, we may be able to state $x_k \neq x_{j^\ast}$ without even querying $\query{j}{k}$. 
\end{itemize}
Our proposed algorithm \texttt{ANNTri} uses all the above ideas to find the nearest neighbor graph of $\X$. For general $\X$, the sample complexity of \texttt{ANNTri} contains a problem-dependent term that involves the order in which the nearest neighbors are found. For an $\X$ consisting of sufficiently well separated clusters, this order-dependence for the sample complexity does not exist. 
\section{Algorithm}\label{sec:alg}
Our proposed algorithm (Algorithm~\ref{alg:ANNTri}) \texttt{ANNTri} finds the nearest neighbor graph of $\X$ with probability $1 - \delta$. 
It iterates over $x_j \in \X$ in order of their subscript index and finds $x_{j^\ast}$ in the $j^{th}$ `round'. 
All bounds, counts of samples, and empirical means are stored in $n \times n$ symmetric matrices in order to share information between different rounds. We use Python array/Matlab notation to indicate individual entries in the matrices, for e.g., $\hd[i,j] = \hd_{i,j}(t)$. The number of $\query{i}{j}$ queries made is queried is stored in the $(i,j)^{th}$ entry of $T$. Matrices $U$ and $L$ record upper and lower confidence bounds on $d_{i,j}$. $U^\triangle$ and $L^\triangle$ record the associated triangle inequality bounds. 
Symmetry is ensured by updating the $(j,i)^{th}$ entry at the same time as the $(i,j)^{th}$ entry for each of the above matrices. 
In the $j^{th}$ round, \texttt{ANNTri} finds the correct $x_{j^\ast}$ with probability $1 - \delta/n$ by calling \texttt{SETri} (Algorithm~\ref{alg:scoob}), a modification of the  successive elimination algorithm for best-arm identification. In contrast to standard successive elimination, 
at each time step \texttt{SETri} only samples those points in the active set 
that have the fewest number of samples.
%
\begin{algorithm}[tb]
   \caption{$\mathtt{ANNTri}$}
   \label{alg:ANNTri}
\begin{algorithmic}[1]
\REQUIRE $n$, procedure $\mathtt{SETri}$, \ref{alg:scoob}, confidence $\delta$
\STATE{Initialize $\hd, T$ as $n\times n$ matrices of zeros, $U, U^\triangle$ as $n \times n$ matrices where each entry is $\infty$, $L, L^\triangle$ as $n\times n$ matrices where each entry is $-\infty$, $\mathrm{NN}$ as a length $n$ array \label{ANNTri:init}}
\FOR{$j=1$ \TO $n$}
\FOR[find tightest triangle bounds]{$i=1$ \TO $n$} \label{alg:tri_start}
\FORALL{$k \neq i$}
\STATE Set $U^\triangle[i,k], \ U^\triangle[k,i], \leftarrow \min_{\ell}\TUi{\ell}{i}{k}$, see \eqref{eq:triUB} \label{alg:triUB_update}
\STATE Set $L^\triangle[i,k], \ L^\triangle[k,i] \leftarrow \max_{\ell}\TLi{\ell}{i}{k}$, see \eqref{eq:triLB} \label{alg:triLB_update}
\ENDFOR
\ENDFOR \label{alg:tri_end}
\STATE $\mathrm{NN}[j] = \mathtt{SETri}(j, \hd, U, U^\triangle, L, L^\triangle, T, \xi = \delta/n)$
\ENDFOR
\RETURN The nearest neighbor graph adjacency list $\mathrm{NN}$
\end{algorithmic}
\end{algorithm}

\begin{algorithm}[tb]
\caption{$\mathtt{SETri}$} \label{alg:scoob}
\begin{algorithmic}[1]
\REQUIRE index $j$, callable oracle $\query{\cdot}{\cdot}$ \eqref{eq:oracle}, six $n\times n$ matrices: $\hd$, $U$, $U^\triangle$, $L$, $L^\triangle$, $T$, confidence $\xi$
\STATE Initialize active set $\A_j \leftarrow \{a \neq j: \max\{L[a,j], L^\triangle[a,j]\} < \min_k \min\{U[j,k], U^\triangle[j,k]\}\}$ \label{scoob:init_A}
\WHILE{$|\A_j| > 1$}
\FORALL[only query points with fewest samples]{$i \in \A_j$ such that $T[i, j] = \min_{k\in \A_j} T[i, k]$}
\STATE Update $\hd[i, j], \ \hd[j, i] \leftarrow (\hd[i, j] \cdot T[i, j] + \query{i}{j})/(T[i, j] + 1)$
\STATE Update $T[i, j], \ T[j, i] \leftarrow T[i, j] + 1$
\STATE Update $U[i, j], \ U[j, i] \leftarrow \hd[i, j] + C_\xi(T[i, j])$
\STATE Update $L[i, j], \ L[j, i] \leftarrow \hd[i, j] - C_\xi(T[i, j])$
\ENDFOR
\STATE {Update $\A_j \leftarrow \{a \neq j: \max\{L[a,j], L^\triangle[a,j]\} < \min_k \min\{U[j,k], U^\triangle[j,k]\}\}$ \label{scoob:update_A}}
\ENDWHILE
\RETURN The index $i$ for which $x_i \in \A_j$
\end{algorithmic}
\end{algorithm}
\subsection{Confidence bounds on the distances}
Using the subGaussian assumption on the noise random process, we can use Hoeffding's inequality and a union bound over time to get the following confidence intervals on the distance $d_{j,k}$:
\begin{equation}\label{eq:C-def}
|\hd_{j,k}(t) - d_{j,k}| \leq \sqrt{2\frac{\log(4n^2(T_{j,k}(t))^2 / \delta)}{T_{j,k}(t)}}
=: C_{\delta/n}(T_{j,k}(t)),
\end{equation}
which hold for all points $x_k \in \X \setminus \{x_j\}$ at all times $t$ with probability $1 - \delta / n$, i.e.
\begin{equation}\label{eq:good-event}
\mathbb{P}(\forall t \in \mathbb{N}, \forall i \neq j, d_{i,j} \in [L_{i,j}(t), U_{i,j}(t)]) \geq 1 - \delta/n, 
\end{equation}
where $L_{i,j}(t) := \hd_{i,j}(t) - C_{\delta/n}(T_{i,j}(t))$ and $U_{i,j}(t) := \hd_{i,j}(t) + C_{\delta/n}(T_{i,j}(t))$. 
\cite{even2006action} use the above procedure to derive the following upper bound for the number of oracle queries used to find $x_{j^\ast}$:
\begin{equation}\label{eq:SE-complexity}
\O\left(\sum_{k \neq j} \frac{\log (n^2/(\delta \Delta_{j,k}))}{\Delta_{j,k}^2}\right),
\end{equation}
where for any $x_k \notin \{x_j, x_{j^\ast}\}$ the suboptimality gap $\Delta_{j,k} := d_{j,k} - d_{j,j^\ast}$ characterizes how hard it is to rule out $x_k$ from being the nearest neighbor. We also set $\Delta_{j,j^\ast} := \min_{k \notin \{j,j^\ast\}} \Delta_{j,k}$. 
Note that one can use tighter confidence bounds as detailed in \cite{garivier2013informational} and \cite{jamiesonN-survey} to obtain sharper bounds on the sample complexity of this subroutine. 

\subsection{Computing the triangle bounds and active set $\A_j(t)$}
Since $\A_j(\cdot)$ is only computed within \texttt{SETri}, we abuse notation and use its argument $t$ to indicate the time counter private to \texttt{SETri}. Thus, the initial active set computed by \texttt{SETri} when called in the $j^{th}$ round is denoted $\A_j(0)$. During the $j^{th}$ round, the active set $\A_j(t)$ contains all points that have not been eliminated from being the nearest neighbor of $x_j$ at time $t$. We define $x_j$'s active set at time $t$ as 
\begin{equation}\label{eq:active_set}
\A_j(t) := \{a \neq j: \max\{L_{a,j}(t), \TL{a}{j}(t)\} < \min_k \min\{U_{j,k}(t), \TU{j}{k}(t)\}\}.
\end{equation}
Assuming $\TL{a}{j}(t)$ and $\TU{j}{k}(t)$ are valid lower and upper bounds on $d_{a,j}, d_{j,k}$ respectively, \eqref{eq:active_set} states that point $x_a$ is active if its lower bound is less than the minimum upper bound for $d_{j,k}$ over all choices of $x_k \neq x_j$. 
Next, for any $(j,k)$ we construct triangle bounds $L^\triangle, U^\triangle$ on the distance $d_{j,k}$. 
Intuitively, for some reals $g,g',h,h'$, if $d_{i,j} \in [g, g']$ and $d_{i,k} \in [h,h']$ then $d_{j,k} \leq g' + h'$, and
\begin{align}\label{eq:intuitive-tri-lb}
d_{j,k} \geq |d_{i,j} - d_{i,k}| = \max\{d_{i,j}, d_{i,k}\} - \min\{d_{i,j}, d_{i,k}\}
\geq (\max\{g, h\} - \min\{g', h'\})_+
\end{align} 
where $(s)_+ {:=} \max\{s, 0\}$. The lower bound can be seen as true by Fig.~\ref{fig:tribounds} in the Appendix. 
Lemma~\ref{lem:tri-bounds} uses these ideas to form upper and lower bounds on distances by the triangle inequality. 
\begin{lemma}\label{lem:tri-bounds}
For all $k \neq 1$, set $\TUi{1}{1}{k}(t) = \TU{1}{k}(t) := U_{1,k}(t)$. For any $i < j$ define 
\begin{align}
\TUi{i}{j}{k}(t) := \min_{\max\{i_1, i_2\}<i}(\min \{U_{i,j}(t), \TUi{i_1}{i}{j}(t) \} + \min \{U_{i,k}(t), \TUi{i_2}{i}{k}(t) \}). \label{eq:triUB}
\end{align}
For all $k \neq 1$, set $\TLi{1}{1}{k}(t) = \TL{1}{k}(t) := L_{1,k}(t)$. For any $i < j$ define
\begin{align}
 \TLi{i}{j}{k}(t) := \max_{\max\{i_1, i_2, i_3, i_4\} < i}& \Big(\max\{L_{i,j}(t), \TLi{i_1}{i}{j}(t), L_{i, k}(t), \TLi{i_2}{i}{k}(t)\}  \nonumber \\
& \;\; - \min\{U_{i,j}(t), \TUi{i_3}{i}{j}(t), U_{i,k}(t), \TUi{i_4}{i}{k}(t)\}\Big)_+, \label{eq:triLB}
\end{align}
where $(s)_+ := \max\{s, 0\}$. If all the bounds obtained by \texttt{SETri} in rounds $i < j$ are correct then 
\begin{equation*}
d_{j,k} \in \big[\TL{j}{k}(t), \TU{j}{k}(t) \big], \quad \text{ where } \quad
\TL{j}{k}(t) := \max_{i < j}\TLi{i}{j}{k}(t) \quad \text{ and } \quad
\TU{j}{k}(t) := \min_{i < j}\TUi{i}{j}{k}(t).
\end{equation*}
\end{lemma}
The proof is in Appendix~\ref{subsec:supp_tri}. \texttt{ANNTri} has access to two sources of bounds on distances: concentration bounds and triangle inequality bounds, and as can be seen in Lemma~\ref{lem:tri-bounds}, the former affects the latter. Furthermore, triangle bounds are computed from other triangle bounds, leading to the recursive definitions of $\TLi{i}{j}{k}$ and $\TUi{i}{j}{k}$. Because of these facts, triangle bounds are dependent on the order in which \texttt{ANNTri} finds each nearest neighbor. 
These bounds can be computed using dynamic programming and brute force search over all possible $i_1, i_2, i_3, i_4$ is not necessary. 
We note that the above recursion is similar to the Floyd-Warshall algorithm for finding  shortest paths between all pairs of vertices in a weighted graph \cite{floyd1962algorithm, cormen2009introduction}. The results in \cite{singla2016actively} show that the triangle bounds obtained in this manner have the minimum $L_1$ norm between the upper and lower bound matrices.

\section{Analysis}\label{sec:analysis}
All omitted proofs of this section can be found in the Appendix Section~\ref{sec:analysis-proofs}. 
\begin{theorem}\label{thm:correctness}
\texttt{ANNTri} finds the nearest neighbor for each point in $\X$ with probability $1-\delta$.
\end{theorem}
\subsection{A simplified algorithm}\label{subsec:relax}
The following Lemma indicates which points must be eliminated initially in the $j^{th}$ round. 
\begin{lemma}\label{lem:easy-elim}
If $\exists i : \ 2U_{i,j} < L_{i,k}$, then $x_k \notin \A_j(0)$ for \texttt{ANNTri}.
\end{lemma}
\begin{proof} \vspace{-1em}
 $2U_{i,j} < L_{i,k} \iff U_{i,j} < L_{i,k} - U_{i,j} \leq \TLi{i}{j}{k}$
\vspace{-1em}
 \end{proof}
Next, we define \texttt{ANNEasy}, a simplified version of \texttt{ANNTri} that is more amenable to analysis. Here, we say that $x_k$ is eliminated in the $j^{th}$ round of \texttt{ANNEasy} if i) $ k{ <} j $ and $\exists i: U_{i,j} < L_{j,k}$ (symmetry from past samples) or ii) $\exists  i: 2U_{i,j} < L_{i,k}$ (Lemma~\ref{lem:easy-elim}). Therefore, $x_j$'s active set for \texttt{ANNEasy} is 
\begin{equation}\label{eq:easy-active}
\A_j = \{a {\neq} j: L_{a,k} \leq 2U_{j,k} \ \forall k \ \text{ and } L_{a,j} < \min_kU_{j,k}\}. 
\end{equation}
To define \texttt{ANNEasy} in code, we remove lines~\ref{alg:tri_start}-\ref{alg:tri_end} of \texttt{ANNTri} (Algorithm~\ref{alg:ANNTri}), and call a subroutine \texttt{SEEasy} in place of \texttt{SETri}. \texttt{SEEasy} matches \texttt{SETri} (Algorithm~\ref{alg:scoob}) except that lines~\ref{scoob:init_A} and \ref{scoob:update_A} are replaced with \eqref{eq:easy-active} instead. We provide full pseudocode of both \texttt{ANNEasy} and \texttt{SEEasy} in the Appendix \ref{subsec:easy_pseudo}. Though \texttt{ANNEasy} is a simplification for analysis, we note that it empirically captures much of the same behavior of \texttt{ANNTri}. In the Appendix~\ref{subsec:anneasy_vs_anntri} we provide an empirical comparison of the two. 

\subsection{Complexity of \texttt{ANNEasy}}\label{subsec:complex}

We now turn our attention to account for the effect of the triangle inequality in \texttt{ANNEasy}. 

\begin{lemma}\label{lem:lower-bound-samples}
For any $x_k \in \X$ if the following conditions hold for some $i < j$, then $x_k \notin \A_j(0)$.
\begin{equation}\label{eq:lower-bound-samples}
6C_{\delta/n}(1) \leq d_{i,k} - 2d_{i,j} \quad \text{and} \quad
\{j,k\} \cap (\cup_{m < i}\{\ell : 2d_{m,i} < d_{m,\ell}\}) = \emptyset.
\end{equation}
\end{lemma}
The first condition characterizes which $x_k$'s must satisfy the condition in Lemma~\ref{lem:easy-elim} for the $j^{th}$ round. The second guarantees that $x_k$ was sampled in the $i^{th}$ round, a necessary condition for forming triangle bounds using $x_i$. 
\begin{theorem}\label{thm:ANNTri-complex}
Conditioned on the event that all confidence bounds are valid at all times, \texttt{ANNEasy} learns the nearest neighbor graph of $\X$ in the following number of calls to the distance oracle:
\begin{align}\label{eq:general-complex}
\mathcal{O}\left( \sum_{j=1}^n \sum_{k > j} \mathds{1}_{[A_{j,k}]} H_{j,k} 
+ \sum_{k<j} \mathds{1}_{[A_{j,k}]}(H_{j,k} - \mathds{1}_{[A_{k,j}]} H_{k,j})_+ \right).
\end{align}
In the above expression $H_{j,k}: = \frac{\log (n^2/(\delta \Delta_{j,k}))}{\Delta_{j,k}^2}$ and $\mathds{1}_{[A_{j,k}]} := 1$, if $x_k$ does not satisfy the triangle inequality elimination conditions of \eqref{eq:lower-bound-samples} $\forall i < j$, and $0$ otherwise. 
\end{theorem}
In Theorem~\ref{thm:notri-complex}, in the Appendix, we state the sample complexity when triangle inequality bounds are ignored by \texttt{ANNTri}, and this upper bounds \eqref{eq:general-complex}. 
Whether a point can be eliminated by the triangle inequality depends both on the underlying distances and the order in which \texttt{ANNTri} finds each nearest neighbor (\textit{c.f.}\ Lemma~\ref{lem:lower-bound-samples}). In general, this dependence on the order is necessary to ensure that past samples exist and may be used to form upper and lower bounds. Furthermore, it is worth noting that even without noise the triangle inequality may not always help. A simple example is any arrangement of points such that $0 < r \leq d_{j,k} < 2r \ \forall j,k$. To see this, consider triangle bounds on any distance $d_{j,k}$ due to any $x_i, x_{i'} \in \X \backslash\{x_j,x_k\}$. Then $|d_{i, j} - d_{i,k}| \leq r < 2r \leq d_{i', j} + d_{i',k} \ \forall i, i'$ so $\TL{i}{j} < \TU{j}{k} \ \forall i,j,k$. Thus no triangle upper bounds separate from triangle lower bounds so no elimination via the triangle inequality occurs. In such cases, it is necessary to sample all $\O(n^2)$ distances. However, in more favorable settings where data may be split into clusters, the sample complexity can be much lower by using triangle inequality. 
\subsection{Adaptive gains via the triangle inequality}\label{sec:active_gains}
\begin{figure}
\centering
\begin{subfigure}{.4\textwidth}
  \centering
  \includegraphics[width=.6\linewidth]{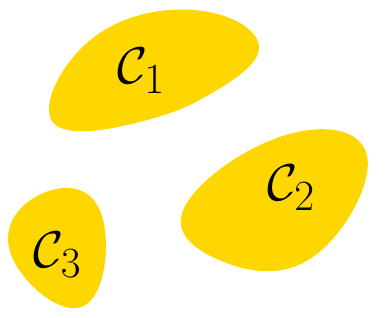}
  \caption{Clustered data}
  \label{fig:clustered_data}
\end{subfigure}%
\begin{subfigure}{.6\textwidth}
  \centering
  \includegraphics[width=.6\linewidth]{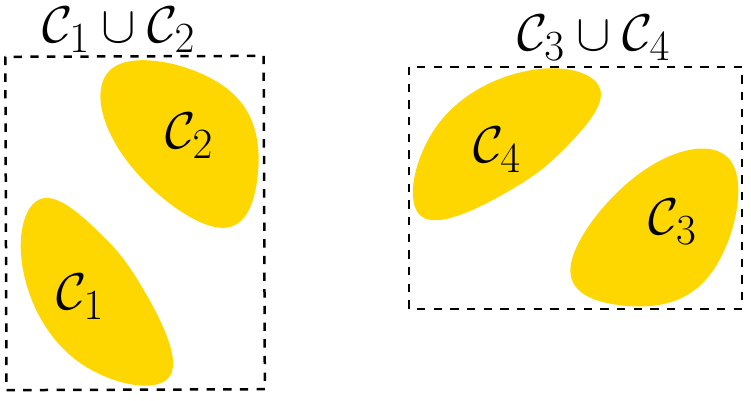}
  \caption{Hierarchical clusters}
  \label{fig:hierarchical_data}
\end{subfigure}
\caption{Example datasets where triangle inequalities lead to provable gains.}
\label{fig:example_data}
\vspace{-1em}
\end{figure}
We highlight two settings where \texttt{ANNTri} provably achieves sample complexity better than $\O(n^2)$ independent of the order of the rounds. Consider a dataset containing $c$ clusters of $n/c$ points each as in Fig.~\ref{fig:clustered_data}. 
Denote the $m^{th}$ cluster as $\C_m$ and 
suppose the distances between the points are such that
\begin{equation}\label{eq:clustered-dataset}
\{x_k: d_{i,k} < 6C_{\delta/n}(1) + 2d_{i,j}\} \subseteq \C_m \ \forall i, j \in \C_m.
\end{equation}
The above condition is ensured if the distance between any two points belonging to different clusters is at least a $(\delta, n)$-dependent constant plus twice the diameter of any cluster. 
\begin{theorem}\label{thm:sqrt_n}
Consider a dataset of $\sqrt{n}$ clusters which satisfy the condition in \eqref{eq:clustered-dataset}. Then \texttt{ANNEasy} learns the correct nearest neighbor graph of $\X$ with probability at least $1-\delta$ in 
\begin{equation}\label{eq:clust-complex-thm}
\O\left(n^{3/2}\overline{\Delta^{-2}}\right)
\end{equation}
queries where $\overline{\Delta^{-2}} := \frac{1}{n^{3/2}}\sum_{i=1}^{\sqrt{n}}\sum_{j, k \in \C_i} \log (n^2/(\delta \Delta_{j,k}))\Delta_{j,k}^{-2}$ is the average number of samples distances between points in the same cluster. 
\end{theorem}
By contrast, random sampling requires $\O(n^2\Delta_{\text{min}}^{-2})$ where $\Delta_{\text{min}}^{-2} := \min_{j,k}\log (n^2/(\delta \Delta_{j,k}))\Delta_{j,k}^{-2} \geq \overline{\Delta^{-2}}$.
In fact, the value in \eqref{eq:general-complex} be be even lower if unions of clusters also satisfy \eqref{eq:clustered-dataset}. In this case, the triangle inequality can be used to separate \emph{groups} of clusters. For example, in 
Fig.~\ref{fig:hierarchical_data}, if $\C_1 \cup \C_2$ and $\C_3 \cup \C_4$ satisfy \eqref{eq:clustered-dataset} along with $\C_1, \cdots, \C_4$, then the triangle inequality can separate $\C_1 \cup \C_2$ and $\C_3 \cup \C_4$. 
This process can be generalized to consider a dataset that can be split recursively into into subclusters following a binary tree of $k$ levels. At each level, the clusters are assumed to satisfy \eqref{eq:clustered-dataset}. 
We refer to such a dataset as \emph{hierarchical in \eqref{eq:clustered-dataset}}. 

\begin{theorem}\label{lem:nlogn}
Consider a dataset $\X = \cup_{i=1}^{n/\nu}\C_i$ of $n/\nu$ clusters of size $\nu = \O(\log(n))$ that is hierarchical in \eqref{eq:clustered-dataset}. 
Then \texttt{ANNEasy} learns the correct nearest neighbor graph of $\X$ with probability at least $1-\delta$ in 
\begin{equation}\label{eq:tree-complex}
\O\left(n\log(n)\overline{\Delta^{-2}}\right)
\end{equation}
queries where $\overline{\Delta^{-2}} := \frac{1}{n\nu}\sum_{i=1}^{n/\nu}\sum_{j, k \in \C_i} \log (n^2/(\delta \Delta_{j,k}))\Delta_{j,k}^{-2}$ is the average number of samples distances between points in the same cluster. 
\end{theorem}
Expression \eqref{eq:tree-complex} matches known lower bounds of $\O(n\log(n))$ on the sample complexity for learning the nearest neighbor graph from noiseless samples \citep{vaidya1989ano}, the additional penalty of $\overline{\Delta^{-2}}$ is due to the effect of noise in the samples. 
In Appendix \ref{subsec:avg_perf}, we state the sample complexity in the average case, as opposed to the high probability statements above. The analog of the cluster condition \eqref{eq:clustered-dataset} there does not involve constants and is solely in terms of pairwise distances (\textit{c.f.}\ \eqref{eq:clustered-dataset-avg}).

\section{Experiments}\label{sec:exps}
Here we evaluate the performance of \texttt{ANNTri} on simulated and real data. To construct the tightest possible confidence bounds for \texttt{SETri}, we use the law of the iterated logarithm as in \cite{jamiesonN-survey} with parameters $\epsilon=0.7$ and $\delta=0.1$. Our analysis bounds the number of queries made to the oracle. We visualize the performance by tracking the empirical \emph{error rate} with the number of queries made per point. For a given point $x_i$, we say that a method makes an error at the $t^{th}$ sample if it fails to return $x_{i^\ast}$ as the nearest neighbor, that is, $x_{i^\ast} \neq \arg\min_j \hat{d}[i, j]$. Throughout, we will compare \texttt{ANNTri} against random sampling. Additionally, to highlight the effect of the triangle inequality, we will compare our method against the same active procedure, but ignoring triangle inequality bounds (referred to as \texttt{ANN} in plots). All baselines may reuse samples via symmetry as well. We plot all curves with $95\%$ confidence regions shaded. 

\subsection{Simulated Experiments}\label{subsec:sims}
We test the effectiveness of our method, we generate an embedding of $10$ clusters of $10$ points spread around a circle such that each cluster is separated by at least $10\%$ of its diameter in $\R^2$ as in shown in Fig.~\ref{fig:sep.1_sig.1:sub1}. We consider Gaussian noise with $\sigma=0.1$. 
In Fig.~\ref{fig:sep.1_sig.1:sub2}, we present average error rates of \texttt{ANNTri}, \texttt{ANN}, and \texttt{Random} plotted on a log scale. \texttt{ANNTri} quickly learns $x_{i^\ast}$ and has lower error with $0$ samples due to initial elimination by the triangle inequality. The error curves are averaged over $4000$ repetitions. All rounds were capped at $10^5$ samples for efficiency. 

\begin{figure}
\centering
\begin{subfigure}{.33\textwidth}
  \centering
  \includegraphics[width=.8\linewidth]{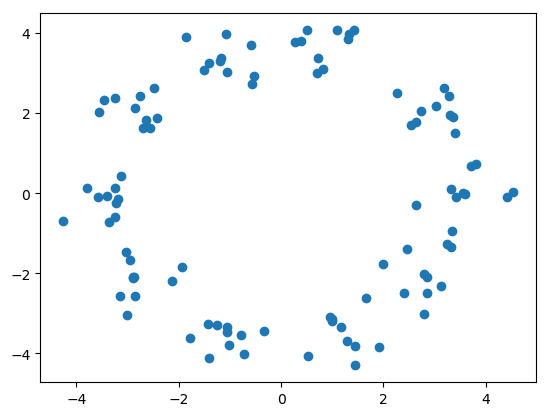}
  \caption{Example embedding}
  \label{fig:sep.1_sig.1:sub1}
\end{subfigure}%
\begin{subfigure}{.33\textwidth}
  \centering
  \includegraphics[width=.8\linewidth]{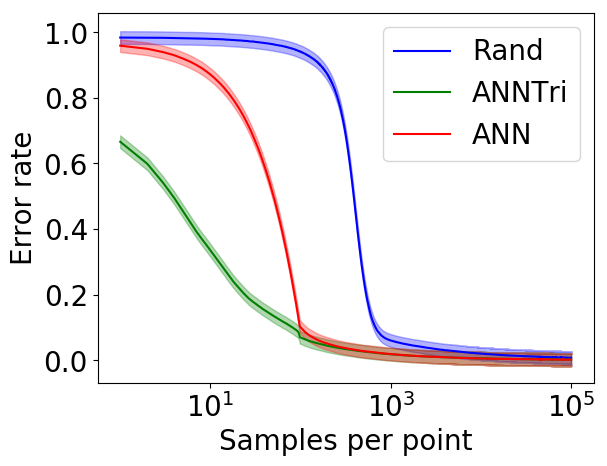}
  \caption{Error curves}
  \label{fig:sep.1_sig.1:sub2}
\end{subfigure}
\begin{subfigure}{.33\textwidth}
  \centering
  \includegraphics[width=.8\linewidth]{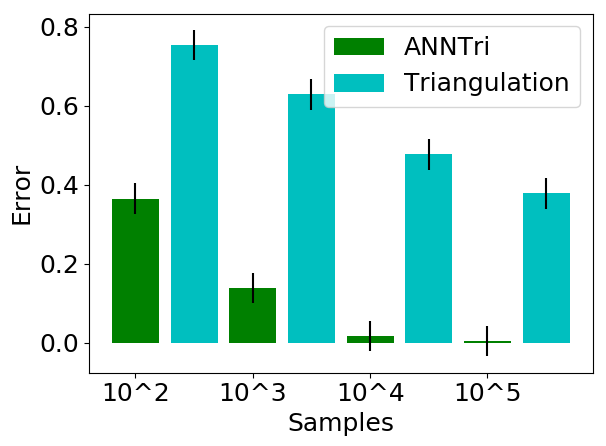}
	\caption{Comparison to triangulation}
  \label{fig:triangulation}
\end{subfigure}
\label{fig:sep.1_sig.1}
\caption{Comparison of \texttt{ANNTri} to \texttt{ANN} and \texttt{Random} for $10$ clusters of $10$ points separated by $10\%$ of their diameter with $\sigma=0.1$. \texttt{ANNTri} identifies clusters of nearby points more easily.}
\vspace{-1em}
\end{figure}


\subsubsection{Comparison to triangulation}
An alternative way a practitioner may use to obtain the nearest neighbor graph might be to estimate distances with respect to a few anchor points and then triangulate to learn the rest. \cite{eriksson2010learning} provide a comprehensive example, and we summarize in Appendix \ref{subsec:tri_supp} for completeness. 
The triangulation method is na\"{i}ve for two reasons. First, it requires \emph{much} stronger modeling assumptions than \texttt{ANNTri}--- namely that the metric is Euclidean and the points are in a low-dimensional of known dimension. Forcing Euclidean structure can lead to unpredictable errors if the underlying metric might not be Euclidean, such as in data from human judgments. Second, this procedure may be more noise sensitive because it estimates squared distances. In the example in Section~\ref{subsec:tri_supp}, this leads to the additive noise being sub-exponential rather than subGaussian. In Fig.~\ref{fig:triangulation}, we show that even in a favorable setting where distances are truly sampled from a low-dimensional Euclidean embedding and pairwise distances between anchors are known exactly, triangulation still performs poorly compared to \texttt{ANNTri}. We consider the same $2$-dimensional embedding of points as in Fig.~\ref{fig:sep.1_sig.1:sub1} for a noise variance of $\sigma=1$ and compare the \texttt{ANNTri} and triangulation for different numbers of samples. 

\subsection{Human judgment experiments}\label{subsec:zappos}
\subsubsection{Setup}
Here we consider the problem of learning from human judgments. 
For this experiment, we used a set $\X$ of $85$ images of shoes drawn from the UT~Zappos50k dataset \cite{finegrained, semjitter} and 
seek to learn which shoes are most visually similar. 
To do this, we consider queries of the form ``between $i$, $j$, and $k$, which two are most similar?''. We show example queries in Figs.~\ref{fig:easy_query} and \ref{fig:hard_query} in the Appendix. Each query maps to a pair of triplet judgments of the form ``is $j$ or $k$ more similar to $i$?''. For instance, if $i$ and $j$ are chosen, then we may imply the judgments ``$i$ is more similar to $j$ than to $k$'' and ``$j$ is more similar to $i$ than to $k$''. We  therefore construct these queries from a set of triplets collected from participants on Mechanical Turk by \cite{heim2015active}. The set contains multiple samples of all $85 {84 \choose 2}$ unique triples so that the probability of any triplet response can be estimated. We expect that $i^\ast$ is most commonly selected as being more similar to $i$ than any third point $k$. We take distance to correspond to the fraction of times that two images $i$, $j$ are judged as being more similar to each other than a different pair in a triplet query $(i, j, k)$. Let $E_{i,k}^j$ be the event that the pair $i,k$ are chosen as most similar amongst $i$, $j$, and $k$. Accordingly, we define the `distance' between images $i$ and $j$ as 
\begin{equation*}
d_{i,j}:= \arg \min_{j \neq i} \bbE_{k\sim\text{Unif}(\X\backslash \{i,j\})}\bbE[\mathds{1}_{E^j_{i,k} } | k]
\end{equation*}
 where $k$ is drawn uniformly from the remaining $83$ images in $\X \backslash \{i,j\}$. For a fixed value of $k$, 
\begin{align*}
\bbE[\mathds{1}_{E^j_{i,k} } | k] = \P(E^j_{i,k} ) = \P(\text{``$i$ more similar to $j$ than to $k$''})\P(\text{``$j$ more similar to $i$ than to $k$''}).
\end{align*}
where the probabilities are the empirical probabilities of the associated triplets in the dataset. This distance is 
a quasi-metric on our dataset as it does not always satisfy the triangle inequality; but  satisfies it with a multiplicative constant: $d_{i,j} \leq 1.47(d_{i,k} + d_{j,k}) \ \forall i,j,k$. Relaxing metrics to quasi-metrics has a rich history in the classical nearest neighbors literature \cite{houle2015rank, tschopp2011randomized, goyal2008disorder}, and \texttt{ANNTri} can be trivially modified to handle quasi-metrics. However, we empirically note that  $<1\%$ of the distances violate the ordinary triangle inequality here so we ignore this point in our evaluation.
\begin{figure}
\begin{subfigure}{.5\textwidth}
  \centering
  \includegraphics[width=.6\linewidth]{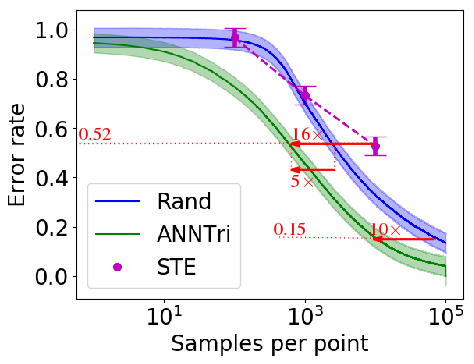}
  \caption{Sample complexity gains}
  \label{fig:active_gains}
\end{subfigure}
\begin{subfigure}{.5\textwidth}
  \centering
  \includegraphics[width=.6\linewidth]{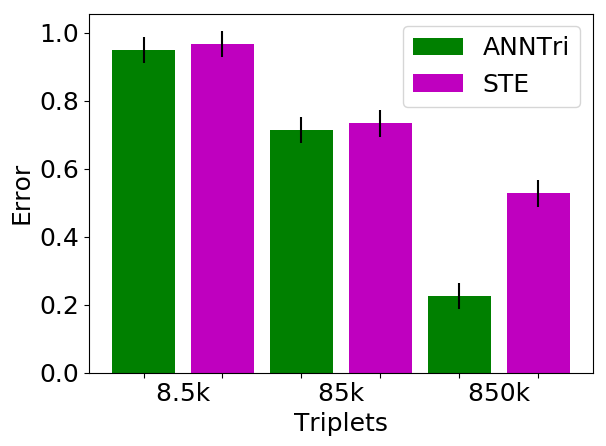}
  \caption{Comparison to STE}
  \label{fig:ordinal_embedding}
\end{subfigure}%
\caption{Performance of \texttt{ANNTri} on the Zappos dataset. \texttt{ANNTri} achieves superior performance over STE in identifying nearest neighbors and has $5-10$x gains in sample efficiency over random.}
\label{fig:zappos_queries}
\vspace{-1em}
\end{figure}

\subsubsection{Results}\label{subsec:zappos_results_main}
When \texttt{ANNTri} or any baseline queries $\query{i}{j}$ from the oracle, we randomly sample a third point $k \in \X \backslash \{i,j\}$ and flip a coin with probability $\P(E^j_{i,k})$. The resulting sample is an unbiased estimate of the distance between $i$ and $j$. 
In Fig.~\ref{fig:active_gains}, we compare the error rate averaged over $1000$ trials of \texttt{ANNTri} compared to \texttt{Random} and \texttt{STE}. We also plot associated gains in sample complexity by \texttt{ANNTri}. In particular, we see gains of $5-10$x over random sampling, and gains up to $16$x relative to ordinal embedding. 
\texttt{ANNTri} also shows $2$x gains over \texttt{ANN} in sample complexity (see Fig.~\ref{fig:zappos_err} in Appendix). 

Additionally, a standard way of learning from triplet data is to perform ordinal embedding. With a learned embedding, the nearest neighbor graph may easily be computed. In Fig.~\ref{fig:ordinal_embedding}, we compare \texttt{ANNTri} against the state of the art STE algorithm \cite{van2012stochastic} for estimating Euclidean embeddings from triplets, and select the embedding dimension of $d=16$ via cross validation. To normalize the number of samples, we first perform \texttt{ANNTri} with a given max budget of samples and record the total number needed. Then we select a random set of triplets of the same size and learn an embedding in $\R^{16}$ via STE. We compare both methods on the fraction of nearest neighbors predicted correctly. On the $x$ axis, we show the total number of triplets given to each method. For small dataset sizes, there is little difference, however, for larger dataset sizes, \texttt{ANNTri} significantly outperforms STE. 
Given that \texttt{ANNTri} is active, it is reasonable to wonder if STE would perform better with an actively sampled dataset, such as \citep{tamuz2011adaptively}. Many of these methods are computationally intensive and lack empirical support \citep{jamieson2015next}, but we can embed using the full set of triplets to mitigate the effect of the subsampling procedure. Doing so, STE achieves $52\%$ error, within the confidence bounds of the largest subsample shown in Fig. \ref{fig:ordinal_embedding}. In particular, more data and more carefully selected datasets, may not correct for the bias induced by forcing Euclidean structure.

\section{Conclusion}
In this paper we solve the nearest neighbor graph problem by adaptively querying distances. Our method makes no assumptions beyond standard metric properties and is empirically shown to achieve sample complexity gains over passive sampling. In the case of clustered data, we show provable gains and achieve optimal rates in favorable settings. 

\clearpage
\bibliographystyle{plainnat}
\bibliography{learning_nn_graphs}

\clearpage
\appendix
\section*{Appendix}\label{sec:supp}
\section{Additional experimental results and details}
\subsection{Differences between  \texttt{ANNTri} and \texttt{ANNEasy}}
\subsubsection{Pseudocode for \texttt{ANNEasy} and \texttt{SEEasy}}\label{subsec:easy_pseudo}
We begin by providing pseudocode for both \texttt{ANNEasy} and \texttt{SEEasy} as described in Section~\ref{subsec:relax} in Algorithms \ref{alg:ANNEasy} and \ref{alg:seeasy}. 
\begin{algorithm}[tb]
   \caption{$\mathtt{ANNEasy}$}
   \label{alg:ANNEasy}
\begin{algorithmic}[1]
\REQUIRE $n$, procedure $\mathtt{SEEasy}$, \ref{alg:seeasy}, confidence $\delta$
\STATE{Initialize $\hd, T$ as $n\times n$ matrices of zeros, $U$ as $n \times n$ matrix where each entry is $\infty$, $L$ as $n\times n$ matrix where each entry is $-\infty$, $\mathrm{NN}$ as a length $n$ array \label{ANNEasy:init}}
\FOR{$j=1$ \TO $n$}
\STATE $\mathrm{NN}[j] = \mathtt{SEEasy}(j, \hd, U, L, T, \xi = \delta/n)$
\ENDFOR
\RETURN The nearest neighbor graph adjacency list $\mathrm{NN}$
\end{algorithmic}
\end{algorithm}

\begin{algorithm}[tb]
\caption{$\mathtt{SEEasy}$} \label{alg:seeasy}
\begin{algorithmic}[1]
\REQUIRE index $j$, callable oracle $\query{\cdot}{\cdot}$ \eqref{eq:oracle}, 4 $n\times n$ matrices: $\hd$, $U$, $L$, $T$, confidence $\xi$
\STATE {Initialize the active set $\A_j \leftarrow \{a {\neq} j: L[a,k] \leq 2U[j,k] \ \forall k \ \text{ and } L[a,j] < \min_kU[j,k]\}$ \label{seeasy:init_A}}
\WHILE{$|\A_j| > 1$}
\FORALL[only query points with fewest samples]{$i \in \A_j$ such that $T[i, j] = \min_{k\in \A_j} T[i, k]$}
\STATE Update $\hd[i, j], \ \hd[j, i] \leftarrow (\hd[i, j] \cdot T[i, j] + \query{i}{j})/(T[i, j] + 1)$
\STATE Update $T[i, j], \ T[j, i] \leftarrow T[i, j] + 1$
\STATE Update $U[i, j], \ U[j, i] \leftarrow \hd[i, j] + C_\xi(T[i, j])$
\STATE Update $L[i, j], \ L[j, i] \leftarrow \hd[i, j] - C_\xi(T[i, j])$
\ENDFOR
\STATE {Update $\A_j \leftarrow \{a {\neq} j : L[a,k] \leq 2U[j,k] \ \forall k \ \text{ and } L[a,j] < \min_kU[a,k]\}$ \label{seeasy:update_A}}
\ENDWHILE
\RETURN The index $i$ for which $x_i \in \A_j$
\end{algorithmic}
\end{algorithm}

\subsubsection{Empirical differences in performance for \texttt{ANNTri} and \texttt{ANNEasy}}\label{subsec:anneasy_vs_anntri}
In Figure \ref{fig:relax_compare} we compare the empirical performance of \texttt{ANNTri} and \texttt{ANNEasy}. We compare their performance in the same setting as Figure \ref{fig:sep.1_sig.1:sub1} with $10$ clusters of $10$ points separated by their at least $10\%$ of their diameter. The curves are averaged over $4000$ independent trials and plotted with $95\%$ confidence regions. As is indicated in the plot, \texttt{ANNEasy} has similar behavior as \texttt{ANNTri}, but achieves slightly worse performance. 

\begin{figure}
\centering
  \includegraphics[width=.5\linewidth]{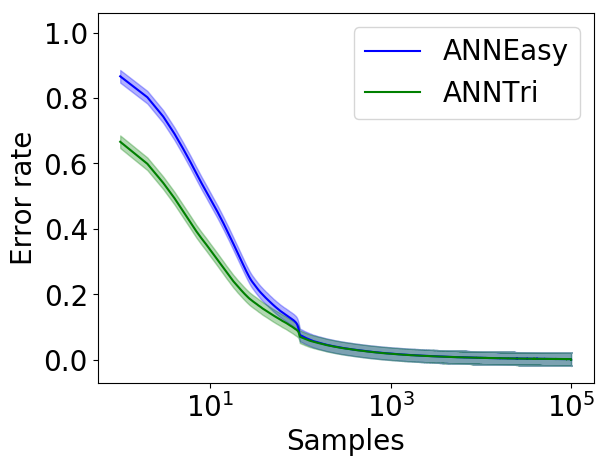}
\caption{Comparison of error in identifying $x_{i^\ast}$ \texttt{ANNTri} and the \texttt{ANNEasy} for $10$ clusters of $10$ points separated by $10\%$ of their diameter with $\sigma=0.1$.}
\label{fig:relax_compare}
\end{figure}


\subsection{Triangulation}\label{subsec:tri_supp}
In this section, we provide a brief review of triangulation to estimate Euclidean embeddings, similar to the presentation in \citep{eriksson2010learning}. The method is summarized as follows. Let $\X$ be a set of $n$ points in Euclidean $d$ space and $\bD$ be the associated Euclidean distance matrix where each entry is the square of the associated Euclidean distance. Let $A$ be a set of anchor points. Without loss of generality, we take $A := \{x_1, \cdots, x_{d+2}\}$. The $+2$ is to correct for the fact that Euclidean distance matrices have rank $d+2$. Let $\bA := \bD[1:d+2, 1:d+2]$ and $\bL := \bD[1:d+2, 1:n]$. Then it can easily be verified that $ \bD = \bL \bA^{-1} \bL^T$. To learn the entries in $\bL$ (as well as $\bA$), sample the distance from each of the $n$ points to the $d+2$ anchors as many times as there is budget for and square the results. The empirical mean is a plugin estimator of the associated entry in $\bL$ and $\bA$, and we take $\widehat{\bL}$ and $\widehat{\bA}$ to be their unbiased estimates. Therefore $\widehat{\bD} :=\widehat{\bL} \widehat{\bA}^{-1} \widehat{\bL}^T$ is an unbiased estimate of $\bD$. With $\widehat{\bD} $, the nearest neighbor graph can easily be computed.  

\subsection{Additional experimental results for Zappos dataset}
In Fig.~\ref{fig:zappos_queries} we show two example queries of the form ``which pair are most similar of these three?''. Some queries are more straightforward whereas some are more subjective. 
\begin{figure}
\centering
\begin{subfigure}{.5\textwidth}
  \centering
  \includegraphics[width=.8\linewidth]{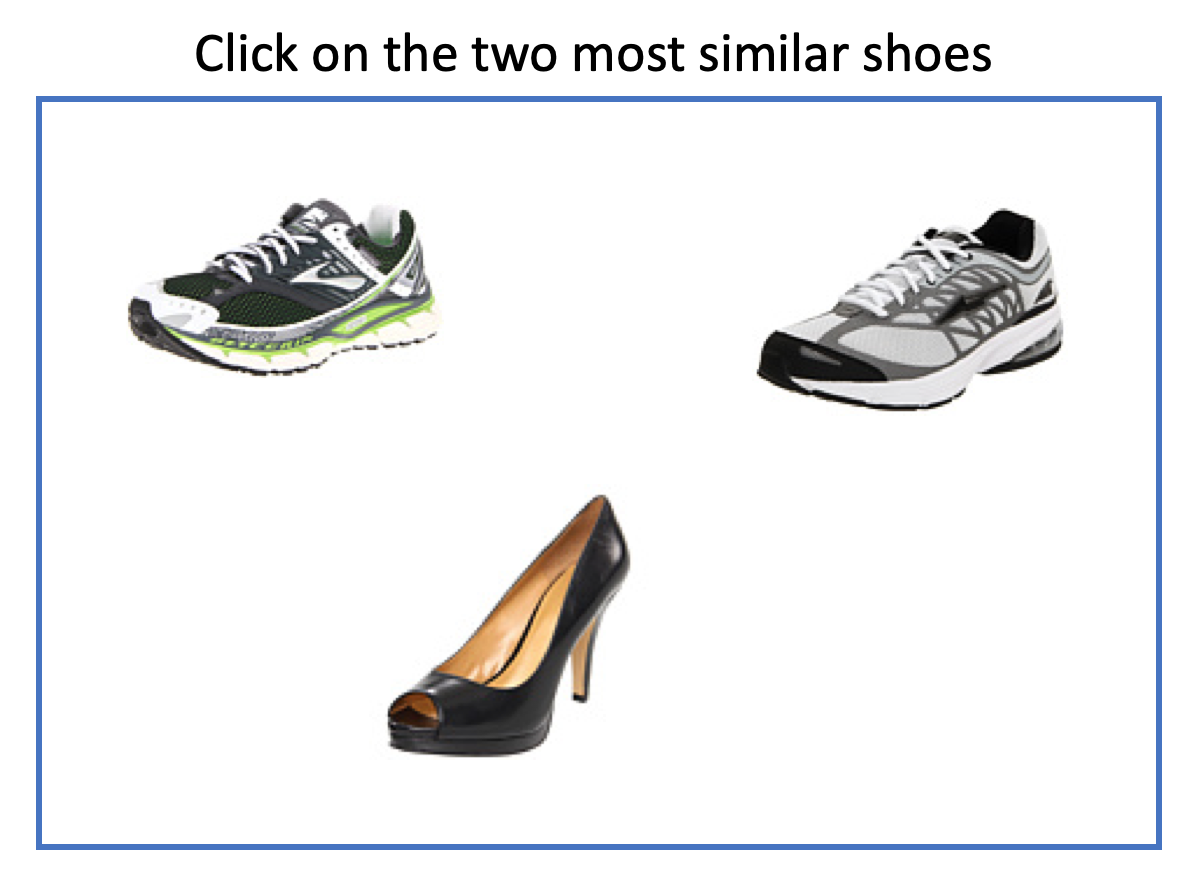}
  \caption{An easy query}
  \label{fig:easy_query}
\end{subfigure}%
\begin{subfigure}{.5\textwidth}
  \centering
  \includegraphics[width=.8\linewidth]{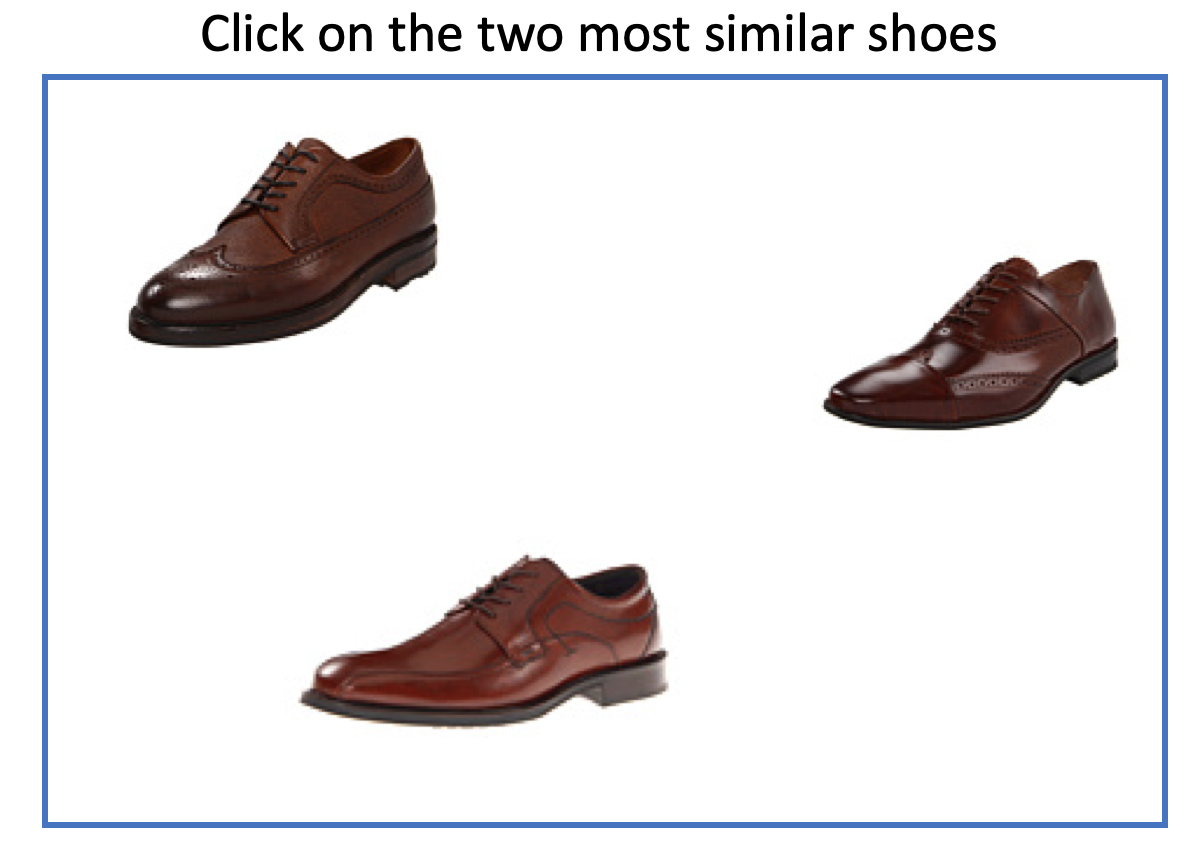}
  \caption{A harder query}
  \label{fig:hard_query}
\end{subfigure}
\caption{Two example zappos queries. }
\label{fig:zappos_queries}
\end{figure}

\begin{figure}
\centering
	\includegraphics[width=0.8\textwidth]{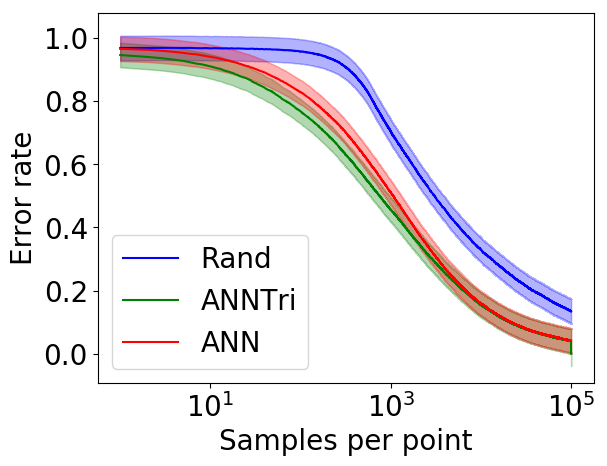}
	\caption{Error rates for nearest neighbor identification on Zappos Data}
	\label{fig:zappos_err}
\end{figure}

Additionally, in Fig.~\ref{fig:zappos_err}, we show the performance of \texttt{ANNTri}, \texttt{ANN}, and \texttt{Random} in identifying nearest neighbors from the Zappos data. In this setting, there is less of an advantage to using the triangle inequality due to the highly noisy and subjective nature of human judgments. Despite this, we still see a slight advantage to \texttt{ANNTri} over \texttt{ANN}. In particular, for moderate accuracy, there is a gain sample complexity of around $2$x. 

\section{Proofs and technical lemmas}\label{sec:analysis-proofs}
\subsection{Proof of Lemma~\ref{lem:tri-bounds}}\label{subsec:supp_tri}
By symmetry for all $i < j$, we have existing samples of $\query{i}{j}$ and $\query{i}{k}$ and we use bounds based on these samples as well as past triangle inequality upper bounds on $d_{i,j}$ and $d_{i,k}$ due to $i_1 < i$ and $i_2 < i$ respectively. The upper bound is derived as follows:
\begin{equation*}
d_{j,k} \leq d_{i,j} + d_{i,k} \leq  \min \{U_{i,j}(t), \TUi{i_1}{i}{j}(t) \} + \min \{U_{i,k}(t), \TUi{i_2}{i}{k}(t) \} =: \TUi{i}{j}{k}
\end{equation*}
Since we may form bounds based on all $i < j$ for which we have both samples of  $\query{i}{j}$ and $\query{i}{k}$, we may optimize over $i$ to get the tightest possible triangle inequality bounds on $d_{j,k}$. 

Lower bounds are derived similarly. Again, intuitively, we may use past samples of both $\query{i}{j}$ and $\query{i}{k}$ and associated bounds to derive a lower bound on $d_{j,k}$. The form is slightly more complicated here since we have to worry about both upper and lower bounds on $d_{i,j}$ and $d_{i,k}$. These bounds may either be from concentration bounds based on past samples directly or past triangle inequality upper and lower bounds on these distances due to points $i_1 - i_4 < i$. 
\begin{align*}
d_{j,k} \geq & |d_{i,j} - d_{i,k}| \\
		 = & \max\{d_{i,j}, d_{i, k}\} - \min\{d_{i,j}, d_{i, k}\} \\ 
		 \geq & (\max \{ \max \{L_{i,j}(t), \TLi{i_1}{i}{j}(t)\}, \ \max \{L_{i, k}(t), \TLi{i_2}{i}{k}(t)\} \}  \\
	& -  \min\{\min\{U_{i,j}(t), \TUi{i_3}{i}{j}(t)\}, \  \min\{U_{i,k}(t), \TUi{i_4}{i}{k}(t)\}\})_+ \\ 
	= & (\max\{L_{i,j}(t), \TLi{i_1}{i}{j}(t), L_{i, k}(t), \TLi{i_2}{i}{k}(t)\}  \\
& - \min\{U_{i,j}(t), \TUi{i_3}{i}{j}(t), U_{i,k}(t), \TUi{i_4}{i}{k}(t)\})_+
\end{align*}
where  $(s)_+ := \max\{s, 0\}$ and $i_1, i_2, i_3, i_4 < i$, (not necessarily unique) are chosen to optimize the bound. Similar to the upper bound, this holds with respect to any $i < j$ and we optimize over $i$. To ease presentation, let $\mathrm{UB}'[i, j] := \min\{U_{i,j}, \min_{l < i} \TUi{l}{i}{j}\}$ and $\mathrm{LB}'[i, j] := \max \{L_{i,j}, \max_{l < i}\TLi{l}{i}{j}\}$ be the tightest upper and lower bounds for $d_{i, j}$. For the lower bound, note that if the argument of $(\cdot)_+$ is negative, then any
\begin{align*}
s &\in [\max\{\mathrm{LB}'[i, j], \mathrm{LB}'[i, k]\}, \min\{\mathrm{UB}'[i, j], \mathrm{UB}'[i, k]\}] \\
&= [\mathrm{LB}'[i, j], \mathrm{UB}'[i, j]] \cap [\mathrm{LB}'[i, k], \mathrm{UB}'[i, k]] \neq \emptyset
\end{align*}
can be the value of both $d_{i, j}$ and $d_{j,k}$ as it lies in both their confidence intervals. Then points $x_j, x_k$ can possibly be at the same location in the metric space, in which case $d_{j, k} = 0$. On the other hand if the RHS is positive, then $x_j$ and $x_k$ cannot be at the same location as $d_{i, j} \neq d_{i, k}$. In fact, the smallest possible value for $d_{j, k}$ occurs if $x_i, x_j, x_k$ are collinear. This can be seen to be true from Figure \ref{fig:tribounds}.
\begin{figure}
\centering
\includegraphics[width=0.7\textwidth]{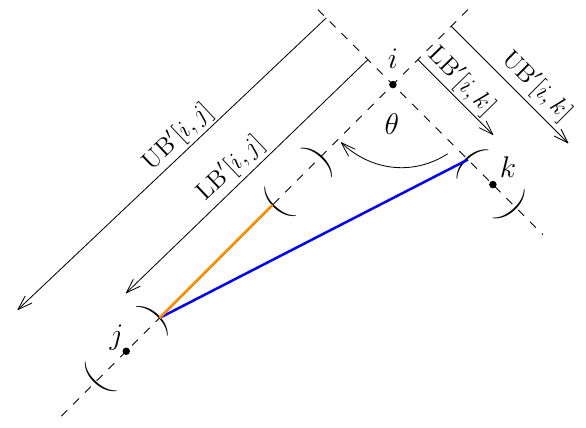}
\caption{Pictorial justification for the lower bound in \eqref{eq:intuitive-tri-lb}. True positions of points $i, j, k$ are shown along with the upper and lower bounds for $d_{i,j}, d_{i, k}$ that are known to the algorithm. If the angle $\theta$ between $\vec{\imath \jmath}$ and $\vec{\imath k}$ is known, the blue segment shows the lowest possible value for $d_{j,k}$ based on the bounds. The orange segment is the value in the RHS of \eqref{eq:intuitive-tri-lb}. Without any information about $\theta$, the three points could be collinear, in which case $d_{j,k}$ could equal the length of the orange segment.  }\label{fig:tribounds}
\end{figure}
We finish with a quick lemma noting what can and cannot be eliminated via the triangle inequality. 
\begin{lemma}\label{lem:closest}
Conditioned on the good event that all bounds are correct at all times, the triangle inequality cannot be used to to separate the two closest points to any given third point. 
\end{lemma}
\begin{proof}
Consider finding $x_{i^\ast}$. Let $d_{i, i^\ast} \leq d_{i,j} \leq d_{i,k} \forall k\neq i^\ast, j$. By the triangle inequality, $d_{i, i^\ast} \leq d_{i,j} + d_{j, i^\ast}$ Clearly, the RHS is no smaller than $d_{i,j}$. Since we are conditioning on all bounds being correct at all times, no upper bound on $d_{i,i^\ast}$ from the triangle inequality can ever be smaller that $d_{i,j}$. Rearranging the inequality, we see that $d_{i, i^\ast} - d_{j, i^\ast} \leq d_{i,j}$. The LHS is no larger than $d_{i, i^\ast}$, and $d_{i, i^\ast} $ is the only distance wrt $x_i$ that is smaller than $d_{i, j}$ by assumption. Therefore, no lower bound on $d_{i,j}$ due to the triangle inequality is greater than $d_{i, i^\ast} < d_{i,j}$. 
\end{proof}

\subsubsection{Helper Lemmas}
\begin{lemma}\label{lem:CI_bounds}
Let $t \in \mathbb{N}$ index the rounds of the procedure \texttt{SETri} in finding $x_{i^\ast}$. Suppose all confidence intervals are valid, i.e., \eqref{eq:good-event} is true. Then $\forall j \neq i$ and all $t$, 
\begin{equation}
L_{i,j}(t) \geq d_{i,j} - 2C_{\delta/n}(T_{i,j}(t)) \quad \text{and} \quad 
U_{i,j}(t) \leq d_{i,j} + 2C_{\delta/n}(T_{i,j}(t)).
\end{equation}
\end{lemma}
\begin{proof}
If the good event \eqref{eq:good-event} is true then for any pair $(i,j)$ and time $t$ we have 
\begin{equation*}
\hd_{i,j}(t) < d_{i,j} + C_{\delta/n}(T_{i,j}(t)) \implies U_{i,j}(t) := \hd_{i,j}(t) + C_{\delta/n}(T_{i,j}(t)) \leq d_{i,j} + 2C_{\delta/n}(T_{i,j}(t)).
\end{equation*}
A similar calculation can be done for $L_{i,j}(t)$ as well.
\end{proof}

\begin{lemma}\label{lem:tri-elim}
 Let $j > i$, and let $t_j$ be the time when $x_j$ is last sampled in the $i^{th}$ round and equivalently for $t_k$. Assume without loss of generality that $d_{i,j} < d_{i,k}$. If $d_{i,j}$ and $d_{i,k}$ are such that
\begin{equation}\label{eq:tri-elim}
4 C_{\delta/n}(T_{i,j}(t_j)) + 2 C_{\delta/n}(T_{i,k}(t_k)) \leq d_{i,k} - 2d_{i,j}
\end{equation}
then \texttt{SETri} can eliminate $d_{j,k}$ without sampling it, i.e., $x_k \notin \A_j(0)$.
\end{lemma}
\begin{proof}
Focusing on the number of $\query{i}{j}$ queries, we have that
\begin{equation}\label{eq:elimtri1}
2U_{i,j}(t_j) = 2(\hd_{i,j}(t_j) + C_{\delta/n}(T_{i,j}(t_j))) 
\leq 2(d_{i,j} + 2 C_{\delta/n}(T_{i,j}(t_j))),
\end{equation}
the inequality in \eqref{eq:elimtri1} is due to Lemma~\ref{lem:CI_bounds}, and using the number of $\query{i}{k}$ queries, 
\begin{equation}\label{eq:elimtri2}
L_{i, k}(t_k) \geq \hd_{i,k}(t_k) - C_{\delta/n}(T_{i,k}(t_k)) \geq d_{i,k} - 2 C_{\delta/n}(T_{i,k}(t_k)).
\end{equation}
The first inequality in \eqref{eq:elimtri2} is because if $k < j$ then there may have been more $\query{k}{i}$ queries beyond the $t_k$ number of $\query{i}{k}$ queries made while finding $x_{i^\ast}$. Rearranging the equation in the Lemma statement,
\begin{equation*}
2d_{i,j} + 4 C_{\delta/n}(T_{i,j}(t_j)) \leq d_{i,k} - 2 C_{\delta/n}(T_{i,k}(t_k)),
\end{equation*}
which implies that $2U_{i,j} \leq L_{i, k}$ from \eqref{eq:elimtri1}, \eqref{eq:elimtri2}. Hence from Lemma~\ref{lem:easy-elim} $x_k \notin \A_j(0)$.
\end{proof}

\begin{lemma}\label{lem:circulant}
There exists a dataset $\mathcal{P}$ containing $2\nu$ points such that for all $x_p \in \mathcal{P}$ and $\alpha > 0$ the set of suboptimality gaps $\Delta_{p,p'}$ is 
\begin{equation}\label{eq:delta-profile}
\left\lbrace 1 - \left( \frac{s-1}{\nu - 1} \right)^\alpha :
s \in \{1,2,\ldots, \nu-1\} \right\rbrace.
\end{equation}
\end{lemma}
\begin{proof}
Note that there are $\nu-1$ values given in \eqref{eq:delta-profile} while there are $2\nu-2$ points in the cluster, excluding $x_p$ and $x_{p^\ast}$. Each value in \eqref{eq:delta-profile} is the suboptimality gap for two distinct points in $\mathcal{P}\setminus\{x_{p}, x_{p^\ast}\}$. We can construct such a dataset $\mathcal{P}$ in the following manner. 

We index these points as $p, p_1, p_2, \ldots, p_{2\nu - 1}$. Suppose the pairwise distance values are such that
\begin{align}
d_{p,p_1} > d_{p,p_2} > \cdots > d_{p,p_{\nu-1}} > d_{p,p_{\nu}} =: d_{p,p^\ast}, &\text{ and }
d_{p, p_{\nu+1}} < d_{p, p_{\nu+2}} < \cdots < d_{p, p_{2\nu - 1}} \text{ such that} \nonumber\\
\forall s \in \{1, 2, \ldots, \nu-1\} \text{ we have that } d_{p,p_{\nu-s}} &= d_{p,p_{\nu+s}} \implies d_{p,p_i} = d_{p,p_{2\nu-i}}. \label{eq:dist-profile}
\end{align}
We can then construct a $2\nu \times 2\nu$ distance matrix $D$ in the following manner. The first row of $D$ is
\begin{equation*}
D[0,:] := \begin{bmatrix}
0 & d_{p,p_1} & d_{p, p_2} & \cdots & d_{p, p_{\nu - 1}} & d_{p, p^\ast} & d_{p, p_{\nu + 1}} & \cdots & d_{p, p_{2\nu - 2}} & d_{p, p_{2\nu-1}}
\end{bmatrix}.
\end{equation*}
The $i$th row of $D$ is obtained by carrying out $i$ circular shifts on the initial row $D[0,:]$ shown above. Thus $D$ is a circulant matrix and we can see $D[i,j]$ and $D[j,i]$ to be as follows.
\begin{equation*}
D[i, j] = \begin{cases}
d_{p, p_{j-i}} & \text{if } j > i\\
d_{p, p_{ 2\nu-(i-j)}} & \text{if } j < i,
\end{cases}
\quad \text{ and } \quad
D[j, i] = \begin{cases}
d_{p, p_{i-j}} & \text{if } i > j\\
d_{p, p_{2\nu-(j-i)}} & \text{if } i < j.
\end{cases}
\end{equation*}
Then using \eqref{eq:dist-profile} we have that $D[i,j] = D[j,i]$ for all $i \neq j$ and the diagonal entries are all $0$. Thus $D$ is symmetric. In addition, the distance values of the points to any point in the cluster take the same set of values. Suppose $d_{p,p^\ast} =: r > 0$ and $d_{p,p_1} = 2r$. 
Choose an $\alpha > 0$ and let 
\begin{equation*}
d_{p, p_{2\nu-i}} = d_{p,p_i} := r\left( 2 - \left( \frac{s-1}{\nu-1}\right)^\alpha\right), ~\forall s \in \{1,2,\ldots, \nu-1\}. 
\end{equation*}
Then $D[i,j] \leq D[i,k] + D[k,j]$ for any three distinct $i,j,k$ as the sum of any two elements is greater than $2r$, which is the largest element in $D$. Thus the distance values in $D$ satisfy the triangle inequality, and $D$ is a valid distance matrix. The suboptimality gaps for any point in the cluster is
$\Delta_{p,p_i} = d_{p,p_i} - d_{p, p^\ast} = r(1 - (\nicefrac{(i-1)}{(\nu-1)})^\alpha)$, choosing $r=1$ finishes the required construction.
\end{proof}
\subsubsection{Proof of Theorem~\ref{thm:correctness}}
\begin{proof}
\texttt{ANNTri} makes an error in finding the nearest neighbor for some point with probability 
$\P(\text{\texttt{SETri} is wrong for some } x_j, j \in \{1,2,\ldots,n\})$. We show that probability is at most $n\xi = \delta$, where the confidence level $\xi$ for each execution of \texttt{SETri} is set to be $\delta/n$. 
We use induction on $s \in \mathbb{N}$ to obtain that 
\begin{equation}\label{eq:ind-hypothesis}
\P(\forall j \in \{1,2,\ldots,s\}, k \neq j, \max\{L_{j,k}(t), \TL{j}{k}(t)\} \leq d_{j,k} \leq \min\{U_{j,k}(t), \TU{j}{k}(t)\}) \geq 1 - s\xi.
\end{equation}
Consider the base case, i.e., point $x_1$. From the initialization of \texttt{ANNTri} \ref{ANNTri:init}, $\min\{U_{1,k}(t), \TU{1}{k}(t)\} = U_{1,k}(t), \min\{L_{1,k}(t), \TL{1}{k}(t)\} = L_{1,k}(t)$ for all $k \neq 1$. 
Using \eqref{eq:good-event} we have $L_{1,k}(t) \leq d_{1,k} \leq U_{1,k}(t)$ with probability $1 - \delta/n$, and since $\xi$ is $\delta/n$ the base case is true. 
Assume the hypothesis \eqref{eq:ind-hypothesis} is true for some $s$. We show that it is true for $s+1$ as well. We can bound the error event as follows.
\begin{align}
&\P(\exists j \in \{1, \ldots, s+1\}, k \neq j : d_{j,k} \notin [\max\{L_{j,k}(t), \TL{j}{k}(t)\}, \min\{U_{j,k}(t), \TU{j}{k}(t)\}]) \label{eq:ind-step}\\
&= \P(\exists j \in \{1, \ldots, s\}, k \neq j : d_{j,k} \notin [\max\{L_{j,k}(t), \TL{j}{k}(t)\}, \min\{U_{j,k}(t), \TU{j}{k}(t)\}]) \nonumber\\
&\quad + \P\bigg(\{k \neq s+1 : d_{s+1,k} \notin [\max\{L_{s+1,k}(t), \TL{s+1}{k}(t)\}, \min\{U_{s+1,k}(t), \TU{s+1}{k}(t)\}]\} \nonumber\\
&\quad \quad \quad \quad \cap \{\forall j \in \{1,2,\ldots,s\}, k \neq j, \max\{L_{j,k}(t), \TL{j}{k}(t)\} \leq d_{j,k} \leq \min\{U_{j,k}(t), \TU{j}{k}(t)\}\}\bigg) \nonumber
\end{align}
From \eqref{eq:ind-hypothesis} the first summand in the RHS of \eqref{eq:ind-step} is at most $s\xi$. In the event corresponding to the second term, all the bounds used by \texttt{SETri} for $d_{j,k}, j \leq s, k \neq j$ are correct. Since $\TU{s+1}{\cdot}$ and $\TL{s+1}{\cdot}$ are both deterministically obtained (see \eqref{eq:triUB}, \eqref{eq:triLB}) from them, they are correct as well. Thus 
\begin{align*}
&\P(\max\{L_{s+1,k}(t), \TL{s+1}{k}(t)\} \leq d_{s+1,k} \leq \min\{U_{s+1,k}(t), \TU{s+1}{k}(t)\})  \\ & = \P(L_{s+1,k}(t) \leq d_{s+1,k} \leq U_{s+1,k}(t)) \geq 1- \xi.
\end{align*}
Hence the second summand in the RHS of \eqref{eq:ind-step} is at most $\xi$. This proves \eqref{eq:ind-hypothesis} for $s+1$ and completes the induction. 

Thus with probability $1 - n\xi = 1-\delta$, the bounds obtained by \texttt{SETri} for finding $x_{j^\ast}, j \in \{1, \ldots, n\}$ are all correct. We show that \texttt{ANNTri} correctly finds all nearest neighbors if the bounds are correct. For if not, suppose \texttt{SETri} returns the wrong nearest neighbor of $x_j$ which happens only if $x_{j^\ast}$ is not the last point in the active set. 
$x_{j^\ast} \notin \A$ because some other point $x_k \in \A$ eliminates it. Then 
$d_{j,k} < \min\{U_{j,k}, \TU{j}{k}\} < \max\{L_{j,j^\ast}, \TL{j}{j^\ast}\} < d_{j,j^\ast}$,
which contradicts the fact that $j^\ast$ is the nearest neighbor.
\end{proof}

\subsubsection{Proof of Lemma~\ref{lem:lower-bound-samples}}
\begin{proof}
Consider a point $x_i, i<j$ which satisfies the first part of \eqref{eq:lower-bound-samples}. If $x_j \in \A_i(0)$ and $x_k \in \A_i(0)$, then neither $x_j$ and $x_k$ were eliminated without sampling 
when \texttt{SEEasyi} was called for $x_i$ and hence $T_{i,j}\geq 1$ and $T_{i,k} \geq 1$. Then we have that
\begin{equation*}
4 C_{\delta/n}(T_{i,j}(t_j)) + 2 C_{\delta/n}(T_{i,k}(t_k)) \leq 6C_{\delta/n}(1) \leq d_{i,k} - 2d_{i,j}
\end{equation*}
and $x_k \notin \A_j(0)$ by Lemma~\ref{lem:tri-elim}. The second part of \eqref{eq:lower-bound-samples} ensures that $\{x_j, x_k\} \subseteq \A_i(0)$ as shown next. The points eliminated from being the nearest neighbor of $x_i$ using triangle inequality are 
$\A_i(0)^\complement = \cup_{m < i}\{\ell : 2U_{m,i} < L_{m,\ell}\}$. If the bounds obtained by \texttt{SEEasy} for all $m < i$ are correct, 
\begin{equation*}
\{\ell : 2U_{m,i} < L_{m,\ell}\} \subseteq 
\{\ell : 2d_{m,i} < d_{m,\ell}\} \implies \A_i(0)^\complement \subseteq \cup_{m < i}\{\ell : 2d_{m,i} < d_{m,\ell}\}.
\end{equation*}
Hence if the second condition of \eqref{eq:lower-bound-samples} is satisfied, then $\{j,k\} \subseteq \A_i(0)$ and we are done.
\end{proof}

\subsubsection{Proof of Theorem~\ref{thm:ANNTri-complex}}
\begin{proof}
Let $x_j$ be the point on which \texttt{SEEasy} is called. Consider the case $j < k$. If $\mathds{1}_{[A_{j,k}]} = 0$ then $x_k \notin \A_j(0)$ and no $\query{j}{k}$ queries are made. Otherwise, $x_k$ can be in the active set and from \eqref{eq:SE-complexity} at most $H_{j,k}$ samples of $d_{j,k}$ are taken. Now consider the case $k < j$. Samples of $d_{j,k}$ are only queried if $x_k \in \A_j(0)$. If $x_j \notin \A_k(0)$, i.e., $x_j$ was eliminated when \texttt{SEEasy} was called for $x_k$ then no $\query{k}{j}$ queries made at that round. Again from \eqref{eq:SE-complexity} at most $H_{j,k}$ samples of $d_{j,k}$ are taken by \texttt{SEEasy} while finding $x_{j^\ast}$. If however $\mathds{1}_{[A_{k,j}]} = 1$, then $\query{k}{j}$ queries were made while finding $x_{k^\ast}$ and let the number of those samples be $\# \query{k}{j}$. Because of the sampling procedure of \texttt{SEEasy}, at most $(H_{j,k} - \# \query{k}{j})_+$ queries are made for $d_{j,k}$. The total number of $\query{j}{k}$ and $\query{k}{j}$ queries is $\max\{H_{j,k}, \# \query{k}{j}\}$, and since $\# \query{k}{j} \leq H_{k,j}$, we get the result.
\end{proof}

\subsection{Details for Section~\ref{sec:active_gains}}
In this section, we consider a case where \texttt{ANNTri} achieves complexity that scales like  $O(n^{1.5})$ as well as $O(n\log(n))$, the known optimal rate for the all nearest neighbors problem for noiseless data. To do this, we first prove a lemma about the complexity of learning with clustered data. In particular, we show that if the data comes from two well separated clusters, then the complexity of learning the nearest neighbor graph can be bounded as the complexity of learning the nearest neighbors of two points looking at the full dataset and the complexity of learning the remaining nearest neighbors graphs on each of the clusters. 

\begin{lemma}\label{lem:two-cluster}
Consider $\X = \C_1 \cup \C_2$ where $\C_1$ and $\C_2$ both satisfy \ref{eq:clustered-dataset} for all $i,j$. Then \texttt{ANNEasy} learns the nearest neighbor graph of $\X$ with probability at least $1 - \delta$ in at most 
\begin{equation}\label{eq:two-clust}
\O\left(|\C_1| + |\C_2| + \H_{\C_1} + \H_{\C_2}\right)
\end{equation}
samples independent of the order in which it finds nearest neighbors where $\H_{\C_i}$ denotes the complexity of learning the nearest neighbor graph of cluster $\C_i$ as bounded by \ref{thm:notri-complex}. 
\end{lemma}
The above lemma implies that for the first point explored in each cluster, it is necessary to look at all other points in the dataset, but for all other points, it is only necessary to search within that point's respective cluster. 
\begin{proof}
Choose a random order of points and fix it. Without loss of generality, we assume that $x_1 \in \C_1$.Let $j_2$ be the first point visited in $C_2$. Throughout, we will ignore reused samples since they only contribute at most a factor of $2$ to the sample complexity as can be seen by Theorems \ref{thm:ANNTri-complex} and \ref{thm:notri-complex} and we seek an upper bound. Via standard analysis for successive elimination, $x_{1^\ast}$ can be be found in $\O\left(\sum_{j=2}^nH_{1,j}\right) = |\C_2| + \O\left(\sum_{j\in \C_1\backslash\{x_1\}}H_{1,j}\right)$ samples with probability at least $1-\delta/n$. For all $i = 2, \cdots, j_2 - 1$, 
\begin{align*}
\A_i(0)^c  \supset  \{ \A_1(0) \cap \{k: d_{i,k} \geq 6d_{i,j} - 3d_{1,1^\ast}\}\}  \supset \{\X \backslash \{x_1\} \cap \C_2\} = \C_2
\end{align*}
which implies that $x_{i^\ast} \in \C_1$. For $x_{j_2}$ we may trivially say that $\A_{j_2}(0)^c \supset \{j_2\}$ so $x_{j_2^\ast}$ can be learned in $\O\left(\sum_{l \neq j_2}H_{i,j}\right) = |\C_1| + \O\left(\sum_{j\in \C_2\backslash\{x_{j_2}\}}H_{j_2,j}\right)$ samples with probability at least $1-\delta/n$. We conclude by showing that for all remaining $x_i$, if $x_i \in \C_1$, then $\A_i(0) \subset \C_1$ and if $x_i \in \C_2$, then $\A_i(0) \subset \C_2$. Consider the case that $x_1 \in \C_1$. Suppose that $\exists x_j \in \A_i(0) \cap \C_2$. Then $2U_{1,i} > L_{1,j}$. 
\begin{align*}
U_{i,1} = \hd_{1,i} + C_{\delta/n}(T_{1,i}) 
\leq d_{1,i} + 2C_{\delta/n}(T_{1,i}) 
= d_{1,i} + 2C_{\delta/n}(1)
\end{align*}
where the first inequality holds by \ref{lem:CI_bounds}. 
Similarly,
\begin{align*}
L_{1,j} = \hd_{1,j} - C_{\delta/n}(T_{1,j}) 
	 \geq d_{1,j} - 2C_{\delta/n}(T_{1,j}) 
	\geq d_{1,j} - 2C_{\delta/n}(1) 
\end{align*}
Then $2( d_{1,i} + 2C_{\delta/n}(1)) \geq 2U_{1,i}  > L_{1,j} \geq d_{1,j} - 2C_{\delta/n}(1)  \implies d_{1,j} < 2d_{1,i} + 6C_{\delta/n}(1) \implies j \in \C_1$ which is a contradiction. A similar proof holds for $x_i \in \C_2$. It remains to argue that $j_2$ can be any number between $2$ (by assumption that $x_1 \in \C_1$) and $|\C_1| + 1$ without affecting the bound on the complexity. By the assumption that $\C_1$ and $\C_2$ satisfy \ref{eq:clustered-dataset}, out of cluster points can be eliminated in a single sample. Therefore, for any $j_2$, $\sum_{l \in \C_1}H_{j_2,l} = |\C_1|$. Then we have that the total complexity is $\O\left(|\C_1| + |\C_2| + \H_{\C_1} + \H_{\C_2}\right) \ \forall j_2$. Since we have considered general orders of finding each nearest neighbor, we are done. 
\end{proof}

\subsubsection{Proof of Theorem~\ref{thm:sqrt_n}}
\begin{proof}
By assumption, the dataset $\X = \cup_{i=1}^c\C_i$ with each cluster satisfies Equation~\ref{eq:clustered-dataset}. Therefore, for all $m$, $\X = \C_m \cup (\cup_{j \neq m} \C_i)$. By applying Lemma~\ref{lem:two-cluster}, iteratively, we bound the complexity in terms of the the complexity of learning the nearest neighbor graph of $\C_m$, the complexity of learning the nearest neighbor graph of $\cup_{j \neq m} \C_i$, and an additive penalty of $n$ which accounts for the samples taken between the two. Since $\X$ is a union of $c$ clusters, this process may repeat $c$ times. Therefore the total complexity can be bounded as 
\begin{equation*}
\O\left(cn + \sum_{i=1}^c \sum_{j, k \in \C_i } H_{j,k}\right)
\end{equation*}
Taking $c = \sqrt{n}$, we see that the above sum is $\O\left(n^{1.5} \overline{\Delta^{-2}}\right)$ where $\overline{\Delta^{-2}} = \frac{1}{c*n} \sum_{i=1}^c \sum_{j, k \in \C_i } H_{j,k}$ is the average number of times intra-cluster distances are sampled. By contrast, the complexity for random sampling is $\O(n^2\Delta_{\text{min}}^{-2})$ where $\Delta_{\text{min}}^{-2} := \min_{j,k}H_{j,k}$. Comparing the two, we see that the latter is larger by at least a factor of $\O(\sqrt{n})$. 
\end{proof}

\subsubsection{Proof of Lemma~\ref{lem:nlogn}}
Next we use Lemma~\ref{lem:two-cluster} to show that for datasets such that the clusters nest, we can achieve complexity scaling in $\O(n\log(n)\overline{\Delta^{-2}})$. In particular, we will recursively apply Lemma~\ref{lem:two-cluster} to show that clusters can be broken into subclusters and initial active sets shrink in diadic splits. 
\begin{proof}
Before we prove the theorem, we begin by introducing some notation to make this proof concise. Recall that we have assumed that $\X$ can be written as a hierarchy of clusters and sub clusters that form a balanced tree. We will denote the root of the tree with the full dataset as the $0^{th}$ level and each split in that level with be indexed by $i = 1, \cdots, 2^\ell$ where $\ell = 0, \cdots, \log(n/\nu) - 1$ denotes the level. For notational ease, we take $\C_{0,1} \equiv \X$. $\C_{\ell, i}$ denotes the $i^{th}$ cluster at the $\ell^{th}$ level of the tree which may be split into subclusters if $\ell < \log(n/\nu) - 1$. The idea will be to traverse the tree and split clusters into subclusters while keeping track of the number of between cluster samples that were be necessary due to the bound in Lemma~\ref{lem:two-cluster}. We let $\H_{\C_{\ell, i}} $ denote complexity of learning the nearest neighbor graph of $\C_{\ell, i}$. 

Randomize the order and fix it. We will proceed by recursively applying Lemma~\ref{lem:two-cluster} to bound the complexity of learning the full nearest neighbor graph of a cluster in terms of learning it for each subcluster plus an additive penalty. By Lemma~\ref{lem:two-cluster} the complexity of finding the nearest neighbor graph of $\X$ can be upper bounded as 
$$\O\left(|\C_{1,1}| + |\C_{1,2}| + \H_{\C_{1,1}} + \H_{\C_{1,2}}\right) = \O\left(n + \H_{\C_{1,1}} + \H_{\C_{1,2}}\right).$$
We may again apply Lemma~\ref{lem:two-cluster} to $\C_{1,1}$ and $\C_{1,2}$. to bound their complexities as $ \O\left(\frac{n}{2} + \H_{\C_{2,1}} + \H_{\C_{2,2}}\right)$ and $ \O\left(\frac{n}{2} + \H_{\C_{2,3}} + \H_{\C_{2,4}}\right)$ respectively where $\C_{1,1} = \C_{2,1} \cup \C_{2,2}$ and $\C_{1,2} = \C_{2,3} \cup \C_{2,4}$. Therefore, similar to the above level, the total additive penalty for samples between clusters is $n$ for the level. We may continue this process of splitting and paying the penalty of $n/2^\ell \times 2^\ell$ between cluster samples down to the bottom level $\ell = \log(n/\nu)$ with clusters of size $\nu$.  

Therefore, we may write the complexity as 
\begin{align}
\O\Bigg(n \log\left(\frac{n}{\nu} \right) 
	+ \sum_{i=1}^{n/\nu}\sum_{\substack{j, k \in \C_{\log(n/\nu), i}}} H_{j,k} \Bigg).
\end{align}
Ignoring logarithmic factors, each complexity term $H_{j,k}$ is of the order $\O(\Delta_{j,k}^{-2})$. Therefore the entire summation is of the order 
$$\O\left(n \log\left(\frac{n}{\nu}\right) + n\nu\overline{\Delta^{-2}}\right)$$
where $\overline{\Delta^{-2}} := \frac{1}{n\nu}\sum_{i=1}^{n/\nu}\sum_{j, k \in \C_i} \log (n^2/(\delta \Delta_{j,k}))\Delta_{j,k}^{-2}$ is the average complexity. Recalling that $\nu = \O(\log(n))$, we are done. 
\end{proof}

\subsection{Sample complexity without using triangle inequality}
\begin{theorem}\label{thm:notri-complex}
With probability $1-\delta$, the number of oracle queries made by \texttt{ANNTri} and \texttt{ANNEasy} if all triangle bounds are ignored is at most
\begin{align}\label{eq:AllNN-complexity}
\O\left(\sum_{i < j} \max\left\{ \frac{\log (n^2/(\delta \Delta_{i, j}))}{\Delta_{i, j}^2} ,  \frac{\log (n^2/(\delta \Delta_{j, i}))}{\Delta_{j, i}^2} \right\}\right).
\end{align}
\end{theorem}
In the experiments, the process of using \texttt{ANNTri} and ignoring triangle inequality bounds is referred to as \texttt{ANN}. 
\begin{proof}
In the case that triange bounds are ignored, \texttt{ANNTri} and \texttt{ANNEasy} are the same. 
Consider the $i^{th}$ round where we seek to identify $x_{i^\ast}$ with probability  $1-\delta/n$. \texttt{ANNTri} has found $x_{\ell^\ast}$ for all $\ell < i$, in particular, it has evaluated $\hd_{\ell, i}, U_{\ell, i}, L_{\ell, i}$. 
For every $x_j \neq x_{i^\ast}, x_j \in \A_i(0)$, we can bound the number of $\query{i}{j}$ queries in the following manner. Suppose $j > i$ and $i^\ast > i$, so that at the beginning of the $i^{th}$ round we have that $T_{i,j} = T_{i,i^\ast} = 0$. 
From \eqref{eq:good-event}, with probability $1-\delta/n$, $x_{i^\ast}$ is the last point in the active set. The point $x_j$ is eliminated from the active set at time $t_j$ if the following is true.
\begin{align}
U_{i,i^\ast}(t_j) \overset{\text{(a)}}{\leq} d_{i,i^\ast} + 2C_{\delta/n}(T_{i^\ast}(t_j)) 
&< d_{i,j} - 2C_{\delta/n}(T_{j}(t_j)) \overset{\text{(b)}}{\leq} L_{i,j}(t_j), \nonumber \\
\implies 4C_{\delta/n}(t_j) &< d_{i,j} - d_{i,i^\ast} = \Delta_{i,j}. \label{eq:worstcase_C-width}
\end{align}
Inequalities (a), (b) are due to Lemma~\ref{lem:CI_bounds}, and 
the fact that if $j$ is eliminated at time $t_j$, then $T_{i,j}(t_j)=t_j$. From the property of the $C_{\delta/n}(\cdot)$ function, \eqref{eq:worstcase_C-width} is ensured when the number of samples of $d_{i,j}$ is
\begin{equation*}
t_j \leq \left\lceil \kappa \frac{\log(n^2/(\delta \Delta_{i,j}/4))}{(\Delta_{i,j}/4)^2} \right\rceil.
\end{equation*}
We now consider the cases when at least one of $i^\ast, j$ are less than $i$. 

$i^\ast > i, j < i$: 
In this case, at the beginning of the $i^{th}$ round $T_{i,j}$ 
is equal to the number of $\query{j}{i}$ queries made (denoted as $\# \query{j}{i}$) while finding $x_{j^\ast}$:
\begin{equation*}
\# \query{j}{i} \leq \left\lceil \kappa \frac{\log(n^2/(\delta \Delta_{j,i}/4))}{(\Delta_{j,i}/4)^2} \right\rceil.
\end{equation*}
If $\# \query{j}{i}> t_j$, then no further $\query{i}{j}$ queries are made in the $i^{th}$ round, as argued next. Because the sampling procedure of \texttt{SETri} queries all points who have the minimum number of samples at current time, if a query $\query{i}{j}$ is made at time $t+1$, that implies $T_{i,i^\ast}(t) = \# \query{j}{i}$. But then $j$ is not in the active set at time $t$ as
\begin{align*}
U_{i,i^\ast}(\# \query{j}{i}) < U_{i,i^\ast}(t_j) < d_{i,j} - 2C_{\delta/n}(t_j) < d_{i,j} - 2C_{\delta/n}(\# \query{j}{i}) = L_{i,j}(\# \query{j}{i})
\end{align*}
and hence $\query{i}{j}$ is not made. If $\# \query{j}{i} < t_j$, then $x_j$ is eliminated when $t_j - \# \query{j}{i}$ more samples of $d_{i,j}$ have been queried. Thus the total number of samples of $d_{i,j}$ is at most $\max\{t_j, \# \query{j}{i}\}$.

The other two cases of 1) $i^\ast < i, j > i$, and 2) $i^\ast < i, j < i$ can be handled similarly.
\end{proof}
\section{Average case performance of \texttt{ANNEasy}}\label{subsec:avg_perf}
We can obtain a different expression for the number of oracle queries if all the random quantities during a run of the algorithm take their expected values. 
In particular, Lemma~\ref{lem:lower-bound-samples} can be relaxed to the following.
\begin{lemma}\label{lem:lower-bound-samples-avg}
If all bounds obtained by \texttt{SEEasy} are correct and all the random quantities take their expected values, then for some $i < j$ such that $x_j \neq x_{i^\ast} \neq x_k$ if we have that
\begin{equation}
d_{i,k} > 6d_{i,j} - 3d_{i,i^\ast} , \quad \text{ and } \quad 
\{j,k\} \cap (\cup_{m < i}\{\ell : 2d_{m,i} < d_{m,\ell}\}) = \emptyset,
\end{equation}
then $2U_{i,j} < L_{i,k}$ 
and hence $x_k \notin \A_j(0)$. 
\end{lemma}
\begin{proof}
In the good event, the point $x_{i^\ast}$ is the last element in the active set $\A_i$ and points $x_j, x_k$ have been eliminated from $\A_i$ at some prior times $t_j, t_k$ respectively. Both $t_j > 0$ and $t_k > 0$ as $\{x_j, x_k\} \subset \A_i(0)$ is ensured by the second part of the condition, as shown in the proof of Lemma~\ref{lem:lower-bound-samples}. At time $t_j$, we have that 
\begin{equation}
\min_{\ell} \hd_{i,\ell} + C_{\delta/n}(t_j) \leq \min_{\ell} U_{i,\ell} \leq L_{i,j} \leq \hd_{i,j} - C_{\delta/n}(t_j).
\end{equation}
If all the random quantities take their expected values, then $\hd_{i,\ell} = d_{i,\ell} \forall \ell \neq i$ and we have that 
\begin{equation}
d_{i,i^\ast} + C_{\delta/n}(t_j) \leq d_{i,j} - C_{\delta/n}(t_j)
\implies C_{\delta/n}(t_j) \leq \Delta_{i,j}/2.
\end{equation}
Under the assumption, $\hd_{i,j} = d_{i,j}$ and using the definition of its upper and lower confidence bounds, we get that
$\bbE [L_{i,j}] \geq d_{i,j} - \Delta_{i,j}/2$ and $\bbE [U_{i,j}] \leq  d_{i,j} + \Delta_{i,j}/2$. Similar bounds are true for $x_k$. Then
\begin{align*}
d_{i,k} > 6d_{i,j} - 3d_{i,i^\ast} 
\implies d_{i,k} - \frac{d_{i,k}-d_{i,i^\ast}}{2}  d_{i,k}-\frac{\Delta_{i,k}}{2} \\
> 2\left(d_{i,j} + \frac{d_{i,j} - d_{i,i^\ast}}{2}\right) = 2\left(d_{i,j} + \frac{\Delta_{i,j}}{2}\right),
\end{align*}
which implies that $L_{i,k}=\bbE[L_{i,k}] > 2\bbE[U_{i,j}] = 2U_{i,j}$ and $x_k \notin \A_j(0)$. 
\end{proof}
If all the random quantities take their expected value, then using Lemma~\ref{lem:lower-bound-samples-avg} and the elimination criterion of \texttt{ANNEasy} (Lemma~\ref{lem:easy-elim}), the complement of the initial active set $\A_j(0)$ (also called the elimination set) can be characterized in the following manner.
\begin{align}
\A_j(0)^\complement &= \cup_{i< j: j\in \A_{i}(0)} \{ \A_{i}(0) \cap \{k: 2U_{i, j} < L_{i, k}\} \} \nonumber\\
&\supseteq \cup_{i< j: j\in \A_{i}(0)} \{ \A_{i}(0) \cap \{k: d_{i,k} > 6d_{i,j} - 3d_{i,i^\ast}\} \}. \label{eq:expect-elim-subset}
\end{align}
Replacing the indicator $\mathds{1}_{[A_{j,k}]}$ in Theorem~\ref{thm:ANNTri-complex} with an indicator for the non-membership of point $x_k$ in the set \eqref{eq:expect-elim-subset} gives us an upper bound to the sample complexity of \texttt{ANNEasy} that is valid when all random quantities take their expected values.

To gain an idea of the savings achieved by our algorithm in comparison to the random sampling, we evaluate the sample complexity expressions for an example dataset. The dataset we look at consists of $c$ clusters, each cluster containing $n/c > 1$ points. The points are indexed such that the $m$th cluster is $\C_m := \{x_{\underline{m}}, x_{1+\underline{m}}, \ldots, x_{\overline{m}}\}$, where
 \begin{align}\label{eq:m_def}
 \underline{m} & := 1+(m-1)n/c \\ 
& \text{and} \nonumber \\
 \overline{m} & := mn/c
 \end{align}
  for all $m \in [c]$. Suppose the distances between the points are such that for any pair $\{x_i, x_j\} \subseteq \C_m$, the set of points 
\begin{equation}\label{eq:clustered-dataset-avg}
\{x_k: d_{i,k} < 6d_{i,j} - 3d_{i,i^\ast}\} \subseteq \C_m.
\end{equation}
The above condition is ensured if the smallest distance between two points belonging to different clusters is at least six times the diameter of any cluster. 
\begin{lemma}
Consider a dataset which satisfies the condition in \eqref{eq:clustered-dataset-avg}. If all random quantities take their expected values, \texttt{ANNEasy} uses $O(\sqrt{n})$ fewer oracle queries than the random sampling baseline to learn the nearest neighbor graph. 
\end{lemma}
\begin{proof}
In the following we assume that all random quantities take their expected values. 
We can find the points that are definitely eliminated using the triangle inequality when \texttt{ANNEasy} is called using \eqref{eq:expect-elim-subset}. The elimination set $\A_1(0)^\complement=\{x_1\}$. For a point $x_i \in \C_1 \notin \E_1$, from \eqref{eq:expect-elim-subset}, \eqref{eq:clustered-dataset-avg} we get that
\begin{equation*}
\A_i(0)^\complement \supseteq \{\A_1(0) \cap \{k: d_{1,k} > 6d_{1,i} - 3d_{1,1^\ast}\}\} 
\supseteq \{(\mathcal{X}\setminus \{x_1\}) \cap \C_1^\complement \} = \C_1^\complement . 
\end{equation*} 
Thus $\A_i(0) \subseteq \C_1$ for all $x_i \in \C_1$. 
Point $x_{\underline{m}}$ 
is the first point processed by \texttt{ANNEasy} in the $m$th cluster. 
Suppose there exists a point $x_{j} \in \C_m \cap \A_{\underline{m}}(0)^\complement$, we show next that leads to a contradiction. Since $x_{j} \notin \A_{\underline{m}}(0)$, there is a point $x_i \in \C_{m'}$ with $i < j, m' < m$ such that $2U{i,\underline{m}} < L_{i,j}$. 
Let $\text{Diam}(\C_m) := \max_{x_\ell,x_{k} \in \C_m}d_{\ell,k}$ be the diameter of cluster $\C_m$ (similarly for $\C_{m'}$) and let $D(\C_{m'}, \C_m) := \min_{x_\ell \in \C_{m'}, x_k \in \C_m} d_{\ell,k}$ be the minimum inter-cluster distance.
Since the random quantities take their expected values, we have that
\begin{align*}
U_{i,\underline{m}} &\geq d_{i,\underline{m}} + \frac{d_{i,\underline{m}} - d_{i,i^\ast}}{2} 
\implies 2U_{i,\underline{m}} \geq 3D(\C_{m'}, \C_m) - \text{Diam}(\C_{m'}),\\
L_{i,j} &\leq d_{i,j} - \frac{d_{i,j} - d_{i,i^\ast}}{2}
\implies L_{i,j} \leq  \frac{\text{Diam}(\C_{m'}) + D(\C_{m'}, \C_m) + \text{Diam}(\C_m)}{2} + \frac{\text{Diam}(\C_{m'})}{2}.
\end{align*} 
Using $2U_{i,\underline{m}} < L_{i,j}$ with the above inequalities 
implies that $2.5D(\C_{m'}, \C_m) < 2D(\C_{m'}) + 0.5D(\C_m)$, which is a contradiction as from \eqref{eq:clustered-dataset-avg} we have that $D(\C_{m'}, \C_m) \geq \max\{3D(\C_{m'}), 3D(\C_m)\}$. Thus we have that 
$\C_m \cap \A_{\underline{m}}(0)^\complement = \emptyset$. 
For any $x_{j} \in \C_m, j \neq \underline{m}$ we have that $x_{j} \in \A_{\underline{m}}(0)$ and hence from \eqref{eq:expect-elim-subset},
\begin{equation*}
\A_{j}(0)^\complement \supseteq \{\A_{\underline{m}}(0) \cap \{k: d_{\underline{m},k} > 6d_{\underline{m},j} - 3 d_{\underline{m},\underline{m}^\ast}\}\}
\supseteq \C_m^\complement.
\end{equation*}
Based on the above discussion, we have a lower bound on the number of points present in the elimination set $\A_j(0)^\complement$ for any $x_j \in \C_m$. By choosing the following values for the indicator in \eqref{eq:AllNN-complexity}
\begin{equation*}
\mathds{1}_{[A_{j,k}]} = \begin{cases}
0 & \text{if } x_j \in \C_m\setminus \{x_{\underline{m}}\} \text{ and } x_k \notin \C_m,\\
1 & \text{otherwise},
\end{cases}
\end{equation*}
we get the following upper bound to the number of oracle queries, where $x_{\overline{m}}$ is the last point in $\C_m$. 
\begin{align}\label{eq:sample-complex-clustered}
\mathcal{O}\Bigg(
\sum_{m=1}^{c} \Bigg(\sum_{k > \underline{m}} H_{\underline{m}, k} + \sum_{k < \underline{m}} H_{\underline{m}, k} - \sum_{\ell = 1}^{m-1} H_{\underline{m}, \underline{\ell}} + \sum_{\ell = 1}^{m-1} (H_{\underline{m}, \underline{\ell}} - H_{\underline{\ell}, \underline{m}})_+ \nonumber \\
+ \sum_{p>\underline{m}}^{\overline{m}} \sum_{q >p}^{\overline{m}} \max\{H_{p,q}, H_{q,p}\}
\Bigg)\Bigg)
\end{align}
where $\underline{m}$ and $\overline{m}$ are defined in \eqref{eq:m_def} and are functions of $m$. 
The number of terms in the sum above is $\mathcal{O}(cn + (n/c)^2)$. A uniform sampling baseline approach would have $\mathcal{O}(n^2)$ terms in its sample complexity. Letting $c=\sqrt{n}$ gives our result. 
\end{proof}
The above lemma ensures that we have $\mathcal{O}(\sqrt{n})$ fewer terms in the sample complexity expression for \texttt{ANNEasy} compared to random sampling if the dataset satisfies \eqref{eq:clustered-dataset-avg}. 
We can get a more precise characterization of the savings in query complexity in terms of the $\Delta_{p,q}$ values. For instance, using a single-parameter model for the distribution of $\Delta_{p,q}$ as done in \cite{jamieson2013finding}, we can directly use their Corollary~1 in our context. 


\begin{lemma}\label{lem:O(n)-fewer}
Consider a clustered dataset $\X = \cup_{m=1}^{c} \C_m$ whose points satisfy \eqref{eq:clustered-dataset-avg}. Each cluster contains an even number $2\nu := n/c$ of points. For any $m \in [c]$ and $x_j \in \C_m$, suppose the suboptimality gaps $\Delta_{j,k}$ for all $x_k \in \C_m$ take one of the following values, parametrized by an $\alpha > 0$:
\begin{equation}\label{eq:delta-profile-lemma}
1 - \left( \frac{s-1}{\nu - 1} \right)^\alpha, \qquad \text{where} \qquad
s \in \{1,2,\ldots, \nu-1\}.
\end{equation}
Note that there are $\nu-1$ values given in \eqref{eq:delta-profile-lemma} while there are $2\nu-2$ points in the cluster, excluding $x_j$ and $x_{j^\ast}$. Each value in \eqref{eq:delta-profile-lemma} is the suboptimality gap for two distinct points in $\C_m$. Ignoring $\log$-factors, if $\alpha=1$ \texttt{ANNTri} finds all nearest neighbors with probability $1-\delta$ in $O(n(\nu^2 + n))$ calls to the oracle, while uniform sampling requires $O(n^2\nu^2)$ calls for the same guarantee.
\end{lemma}
\begin{proof}
By putting the clusters far from each other, one can see that there exist $\X = \cup_{m=1}^{c} \C_m$ whose points satisfy \eqref{eq:clustered-dataset-avg}. 
Lemma~\ref{lem:circulant} shows by explicit construction that the condition on the suboptimality gaps within each cluster as stated in \eqref{eq:delta-profile-lemma} can also be satisfied. 
Note that \eqref{eq:delta-profile-lemma} is the same parametrization as equation~3 in \cite{jamieson2013finding}. 

Consider the points in the $m$th cluster, i.e., points $x_{\underline{m}}$ through $x_{\overline{m}}$. The elimination set $\A_{\underline{m}}(0)^\complement$ can be the singleton $\{x_{\underline{m}}\}$, but by Lemma~\ref{lem:lower-bound-samples-avg} for all $x_p \in \C_m \setminus \{x_{\underline{m}}\}, \A_p(0)^\complement \supseteq \C_m^\complement$. Finding $x_{p^\ast}$ is a best arm identification problem among points within the cluster $\C_m$. The last term in \eqref{eq:sample-complex-clustered} counts the total number of oracle queries made by \texttt{ANNEasy} to identify the nearest neighbors of all $x_p \in \C_m \setminus \{x_{\underline{m}}\}$. Thus the number of oracle queries made by \texttt{ANNEasy} for identifying $x_{p^\ast}$ is at most $\sum_{q \neq p} H_{p,q}$, while uniform sampling will make $nH_{p,p'}$ queries, where $p' := \arg\min_{q \neq p^\ast} \Delta_{p,q}$. 

Ignoring $\log$-factors, the sample complexity for finding $x_{p^\ast}$ for an $x_p \in \C_m$ by \texttt{ANNEasy} is 
\begin{equation*}
\tilde{\mathcal{O}}\left(\sum_{i=1}^{2\nu-1}\Delta_{p,p_i}^{-2} + \sum_{x_\ell \notin \C_m} \Delta_{p,\ell}^{-2}\right) 
= 
\tilde{\mathcal{O}}\left(\sum_{i=1}^{2\nu-1}\Delta_{p,p_i}^{-2} + n- 2\nu\right).
\end{equation*}
Corollary~1 of \cite{jamieson2013finding} lists the value of that sum for different choices of $\alpha$, for e.g., if $\alpha = 1$ then the sample complexity is $\tilde{\mathcal{O}}(\nu^2 + n - 2\nu)$. On the other hand, for finding $x_{p^\ast}$ uniform sampling would make $\tilde{\mathcal{O}}(n \Delta_{p,p'}^{-2})$, i.e., $\tilde{\mathcal{O}}(n (\nu - 1)^2)$ queries. 
By construction of the dataset, finding the nearest neighbor of each point in $\X$ is equally hard. Thus \texttt{ANNTri} would make $\tilde{\mathcal{O}}(n(\nu^2 + n - 2\nu))$ queries while uniform sampling would take $\tilde{\mathcal{O}}(n^2 \nu^2)$ queries.
\end{proof}
Note that our problem setting is inherently different from the noiseless setting where all $x_{i^\ast}$'s can trivially be learned in ${n \choose 2}$ samples. Due to the presence of noise in our queries, many distances must be repeatedly queried so ${n \choose 2}$ samples is insufficient.

\end{document}